\documentclass{article}
\usepackage[a4paper, total={6in, 8in}]{geometry}

\usepackage[utf8]{inputenc} 
\usepackage[T1]{fontenc}    
\usepackage{booktabs}       
\usepackage{amsfonts}       
\usepackage{nicefrac}       
\usepackage{microtype}      
\usepackage{tikz}
\usetikzlibrary{calc}
\usepackage{bm}
\usepackage{braket}
\usepackage{graphicx}
\usepackage{amsthm}
\usepackage{amssymb}
\usepackage{algorithmic}
\usepackage{adjustbox}
\usepackage[linesnumbered,ruled]{algorithm2e}
\newtheorem{thm}{Theorem}
\newtheorem{lem}[thm]{Lemma}
\newtheorem{defi}[thm]{Definition}

\newtheorem{prop}[thm]{Proposition}
\newtheorem{coro}[thm]{Corollary}
\newtheorem{prob}[thm]{Problem}
\usepackage{mathtools}
\usepackage{bbm}
\usepackage[explicit]{titlesec}
\usepackage{nccmath}
\usepackage{xcolor}
\usepackage{caption}
\usepackage{tkz-tab}
\usepackage{subcaption}
\usepackage{wrapfig}
\usepackage{authblk}
\DeclareMathOperator{\pr}{Pr}

\title{ A Quantum-inspired  Algorithm for General Minimum Conical Hull Problems}

\author[1,*]{Yuxuan Du}
\author[2,*]{Min-Hsiu Hsieh}
\author[1]{Tongliang Liu}
\author[1]{Dacheng Tao}
\affil[1]{UBTECH Sydney AI Centre, School of Computer Science, Faculty of Engeering, University of Sydney, Australia}
\affil[2]{Centre for Quantum Software and Information, Faculty of Engineering and Information Technology, University of Technology Sydney, Australia}
\affil[*]{yudu5543@uni.sydney.edu.au, min-hsiu.hsieh@uts.edu.au}
\begin{document}
\date{}

\maketitle
\begin{abstract}
A wide range of fundamental machine learning tasks that are addressed by  the  maximum a posteriori estimation can be reduced to a  general minimum conical hull problem. The best-known solution to tackle  general minimum conical hull problems is the divide-and-conquer anchoring learning scheme (DCA),  whose runtime complexity is polynomial in size.  However, big data is pushing these polynomial algorithms to their performance limits.  In this paper, we propose a sublinear classical algorithm to tackle  general minimum conical hull problems when the input has stored in a sample-based low-overhead data structure. The algorithm's runtime complexity is polynomial in the rank and  polylogarithmic in size.  The proposed algorithm achieves the exponential speedup over DCA and, therefore,  provides  advantages for high dimensional problems.
\end{abstract}

\section{Introduction}

Maximum a posteriori (MAP) estimation is a central problem in machine and statistical learning \cite{bishop2006pattern,korb2010bayesian}.  The general MAP problem has been proven to be NP hard \cite{shimony1994finding}. Despite the hardness in the general case, there are two fundamental learning models, the matrix factorization and the latent variable model, that enable MAP problem to be solved in polynomial runtime under certain constraints \cite{lawrence2004gaussian,lee1999learning,loehlin1987latent,mnih2008probabilistic,schmidt2009bayesian}. The algorithms that have been developed for these learning models have been used extensively in machine learning with competitive performance, particularly on tasks such as subspace clustering, topic modeling, collaborative filtering, structure prediction, feature engineering, motion segmentation, sequential data analysis, and recommender systems \cite{geman1987stochastic,lee1999learning,mnih2008probabilistic}. A recent study demonstrates that MAP problems addressed by matrix factorization and the latent variable models can be reduced to the general minimum conical hull problem \cite{zhou2014divide}. In particular, the general minimum conical hull problem transforms problems resolved by these two learning models into a geometric problem, whose goal is to identify a set of extreme data points with the smallest cardinality in dataset $\bm{Y}$ such that every data point in  dataset $\bm{X}$ can be expressed as a conical combination of the identified extreme data points. Unlike the matrix factorization and the latent variable models that their optimizations generally suffer  from the local minima, a unique global solution is guaranteed for the general minimum conical hull problem \cite{zhou2014divide}. Driven by the promise of a global solution and the  broad applicability, it is imperative to seek algorithms that can efficiently resolve the general minimum conical hull problem with theoretical guarantees. 

The divide-and-conquer anchoring (DCA) scheme is among the best currently known solutions for addressing general minimum conical hull problems.  The idea is to identify all $k$ extreme rays (i.e., $k$ extreme data points) of a conical hull from a finite set of real data points with high probability   \cite{zhou2014divide}. The discovered extreme rays form the global solution for the problem with explainability and, thus, the scheme generalizes better than conventional algorithms, such as expectation-maximization  \cite{dempster1977maximum}. DCA's strategy is  to decompose the original problem into 
distinct subproblems. Specifically, the original conical hull problem is randomly projected on different low-dimensional hyperplanes to ease computation. Such a decomposition is  guaranteed by the fact that the geometry of the original conical hull is partially preserved after a random projection.  However, a weakness of DCA is that it requires a polynomial runtime complexity with respect to the size of the input. This complexity heavily limits DCA’s use in many practical situations given the number of massive-scale datasets that are now ubiquitous \cite{wu2013data}. Hence, more effective methods for solving general minimum conical hull problems are highly desired.    
   
To address the above issue, we propose an efficient classical algorithm that tackles general minimum conical hull problems in polylogarithmic time with respect to the input size. Consider two datasets $\bm{X}$ and $\bm{Y}$ that have stored in a specific low-overhead data structure, i.e., a sampled-based data structure supports the length-square sampling operations \cite{tang2018quantum}. Let  the maximum rank, the Frobenius norm, and the  condition number of the given two  datasets be $k$, $\|\bm{H}\|_F$, and $\kappa$, respectively. We prove that the runtime  complexity of the our  algorithm  is $\tilde{\mathcal{O}}(k^6 \kappa^{12}\|\bm{H}\|_F^6/\epsilon^6)$ with the tolerable level of error $\epsilon$. 
  The achieved sublinear runtime complexity  indicates that our algorithm has capability to benefit  numerous  learning tasks that can be mapped to the general minimum conical hull problem, e.g., the MAP problems addressed by latent variable models and matrix factorization.

Two core ingredients of the proposed algorithm are the ‘divide-and-conquer’ strategy and the reformulation of the minimum conical hull problem as a sampling problem.  We adopt the `divide-and-conquer' strategy  to  acquire a favorable property from DCA. In particular, all subproblems, i.e., the general minimum conical hull problems on different low-dimensional random hyperplanes, are independent of each other.  Therefore, they can be processed in parallel. In addition, the total number of subproblems is only polynomially proportional to the rank of the given dataset  \cite{zhou2013divide,zhou2014divide}.   
An immediate observation is that the efficiency of solving each subproblem governs the efficiency of tackling the general minimum conical hull problem.   To this end,  our algorithm converts each subproblem into an approximated sampling problem and obtains the solution in sublinear runtime complexity. 
Through  advanced sampling techniques \cite{gilyen2018quantum,tang2018quantum}, the runtime complexity to prepare the approximated sampling distribution that corresponds to each subproblem is  polylogarithmic in the size of input. To enable our algorithm has an end-to-end sublinear runtime,  we propose a general heuristic post-selection method to efficiently sample the solution from the approximated distribution,  whose computation cost is also  polylogarithmic in size. 

Our work creates an intriguing aftermath for    the quantum machine learning community. The overarching goal of quantum machine learning is to develop quantum algorithms that quadratically or even exponentially reduce the runtime complexity of classical algorithms \cite{biamonte2016quantum}. Numerous quantum machine learning algorithms with provably quantum speedups have been proposed in the past decade \cite{harrow2009quantum,kapoor2016quantum,lloyd2014quantum}. However, the polylogarithmic runtime complexity in our algorithm implies that a rich class of quantum machine-learning algorithms do not, in fact, achieve these quantum speedups.   More specifically, if a quantum algorithm aims to solve a learning task that can be mapped to the general minimum conical hull problem, its quantum speedup will collapse. In our examples, we show that quantum speedups collapse for these quantum algorithms:  recommendation system  \cite{kerenidis2017quantum},  matrix factorization \cite{du2018quantum}, and clustering \cite{aimeur2007quantum,lloyd2013quantum,wiebe2014quantum}.
 
\textbf{Related Work.} We make the following comparisons with previous studies in maximum a posteriori estimation and quantum machine learning.
\begin{enumerate}
	\item  The mainstream methods to tackle MAP problems  prior to the study \cite{zhou2014divide} can be separated into two groups. The first group includes expectation-maximization, sampling methods, and matrix factorization \cite{Dempster_1977,Geman1984,Salakhutdinov_2007}, where the learning procedure has severely suffered from local optima. The second group contains the method of moments \cite{belkin2010polynomial}, which has suffered from the large variance and may lead to the failure of final estimation. The study \cite{zhou2014divide} effectively alleviates  the difficulties encountered by the above two groups. However, a common weakness possessed by all existing methods is the polynomial  runtime  with respect to the input size. 
\item There are several studies that collapse the quantum speedups by proposing quantum-inspired classical algorithms \cite{chia2019quantum,chia2018quantum,gilyen2018quantum}. For example, the study \cite{tang2018quantum} removes the  quantum speedup for recommendation systems tasks; the study \cite{tang2018quantum2} eliminates the   quantum speedup for principal component analysis problems; the study \cite{ding2019quantum} collapses the  quantum speedup for support vector machine problem.  The correctness of a branch of quantum-inspired algorithms is validated by study \cite{arrazola2019quantum}. The studies \cite{tang2018quantum} and \cite{ding2019quantum} can be treated as a special case of our result, since both recommendation systems tasks and support vector machine can be efficiently reduced to the general conical hull problems \cite{liu2017fast,zhou2013divide}. In other words, our work is a more general methodology to collapse the speedups achieved by quantum machine learning algorithms. 
\end{enumerate}

The rest of this paper proceeds as follows. A formal outline of the general minimum conical hull problem is given in Section \ref{Sec:MCH_setup}. The computational complexity of our algorithm is explained and analyzed in Section   \ref{Sec:main_algorithm}. Section \ref{sec:corrts} discusses the algorithm’s correctness. We conclude the paper in Section \ref{sec:conl}.

\section{Problem Setup}\label{Sec:MCH_setup}
\subsection{Notations}
Throughout this paper, we use the following notations. We denote $\{1,2,..., n\}$ as $[n]$.  Given a vector ${\bm{v}}\in \mathbb{R}^{n}$,  ${\bm{v}}_i$ or ${\bm{v}}(i)$  represents the $i$-th entry of $\bm{v}$ with $i\in [n]$ and $\|{\bm{v}}\|$ refers to the $\ell_2$ norm of $\bm{v}$ with $\|{\bm{v}}\| =\sqrt{ \sum_{i=1}^n {\bm{v}}_i^2}$.  The notation $\bm{e}_i$ always refers to the $i$-th unit basis vector.  Suppose that ${\bm{v}}$ is nonzero, we define $\mathcal{P}_{{\bm{v}}}$ as a probability distribution in which the index $i$ of ${\bm{v}}$ will be chosen with the probability $\mathcal{P}_{{\bm{v}}}(i)=(|{\bm{v}}_i|/\|{\bm{v}}\|)^2$ for any $i\in[n]$.  A sample from $\mathcal{P}_{{\bm{v}}}$ refers to an index number $i$ of $\bm{v}$, which will be sampled with the probability $\mathcal{P}_{{\bm{v}}}(i)$. Given a matrix $\bm{X}\in \mathbb{R}^{n\times m}$, $\bm{X}_{ij}$ represents the $(i,j)$-entry of  $\bm{X}$ with $i\in[n]$ and $j\in[m]$.  $\bm{X}(i,:)$ and $\bm{X}(:,j)$  represent the $i$-th  row and the $j$-th column of the matrix $\bm{X}$, respectively. The transpose of a matrix $\bm{X}$ (a vector ${\bm{v}}$) is denoted by $\bm{X}^{\top}$ (${\bm{v}}^{\top}$). The Frobenius and spectral norm of $\bm{X}$ is denoted as $\|\bm{X}\|_F$ and $\|\bm{X}\|_2$, respectively.
  The condition number  $\kappa_{\scalebox{.49}{X}}$  of a positive semidefinite $\bm{X}$  is  $\kappa_{\scalebox{.49}{X}}=\|\bm{X}\|_2/\sigma_{\min}(\bm{X})$, where $\sigma_{\min}(\bm{X})$ refers to the minimum singular of $\bm{X}$. $\mathbb{R}_+$  refers to the real positive numbers. Given two sets $\mathcal{A}$ and $\mathcal{B}$, we denote  $\mathcal{A}$ minus $\mathcal{B}$ as $\mathcal{A}\setminus \mathcal{B}$. The cardinality of a set $\mathcal{A}$  is denoted as $|\mathcal{A}|$.  We use  the  notation $\tilde{\mathcal{O}}(k)$ as   shorthand for $\mathcal{O}(k\log(n))$.  

\subsection{General minimum conical hull problem}\label{subsec:GMCHprob}
A cone is a non-empty convex set that is closed under the conical combination of its elements. Mathematically, given a set  $\mathcal{R}=\{r_i\}_{i=1}^k$ with $r_i$ being the ray, a cone formed by $\mathcal{R}$ is defined as ${cone}(\mathcal{R}) =\{ \sum_{i=1}^k\alpha_i r_i:~\forall i\in[k],~ r_i\in \mathcal{R},~\alpha_i\in\mathbb{R}_+\}$, and ${cone}(\mathcal{R})$ is the conical hull of $\mathcal{R}$. 
A ray $r_i$ is an extreme ray (or an \textbf{anchor}) if it cannot be expressed as the conical combination of elements in $\mathcal{R}\setminus r_i$.  A fundamental property of the cone and conical hull is its \emph{separability}, namely,  whether a point in ${cone}(\mathcal{R})$ can be represented as a  conical combination  of {certain subsets} of rays that define ${cone}(\mathcal{R})$. The above separability can be generalized to the matrix form as follows \cite{zhou2014divide}. 
\begin{defi}[General separability condition \cite{zhou2014divide}]\label{def:General Sepa}
Let $\bm{X}\in\mathbb{R}^{n_{\scalebox{.49}{X}}\times m}$ be a matrix with $n_{\scalebox{.49}{X}}$ rows. Let $\bm{Y}\in\mathbb{R}^{n_{\scalebox{.49}{Y}}\times m}$ be a matrix with $n_{\scalebox{.49}{Y}}$ rows, and let $\bm{Y}_{\mathcal{A}}$ be a submatrix of $\bm{Y}$ with rows ${\mathcal{A}}\subset [n_{\scalebox{.49}{Y}}]$.  We say that $\bm{X}$ is separable with respect to $\bm{Y}_{\mathcal{A}}$ if $\bm{X}=\bm{F}\bm{Y}_{\mathcal{A}}$, where $rank(\bm{F})\geq |\mathcal{A}|$. 
\end{defi}
In other words, the general separability condition states that,  $\forall i\in[n_{\scalebox{.49}{X}}]$, we have $\bm{X}_i\in{cone}(\bm{Y}_{\mathcal{A}})$, where the set $\bm{Y}_{\mathcal{A}}=\{\bm{Y}(i;:)\}_{i\in {\mathcal{A}}}.$ 
Under this definition, the general minimum conical hull problem aims to find the minimal set ${\mathcal{A}}$, as the so-called  anchor set, from the rows of $\bm{Y}$. 
\begin{defi}[General minimum conical hull problem \cite{zhou2014divide}]\label{def:minimum_c_hull}
Given two matrices $\bm{X}\in\mathbb{R}^{n_{\scalebox{.49}{X}}\times m}$ and $\bm{Y}\in\mathbb{R}^{n_{\scalebox{.49}{Y}}\times m}$ with $n_X$ and $n_Y$ rows, respectively, 
the general minimum conical hull problem ${\rm{MCH}}(\bm{Y}, \bm{X})$ finds a minimal subset of rows  in $\bm{Y}$ whose conical hull contains $cone(\bm{X})$: 
\begin{equation}\label{eqn:min_CH_orig}
{\rm{MCH}}(\bm{Y}, \bm{X}):= \arg\min\{ |{\mathcal{A}}|: \text{cone}(\bm{Y}_{\mathcal{A}})\supset \text{cone}(\bm{X})\}~, 
\end{equation}
where $\text{cone}(\bm{Y}_{\mathcal{A}})$ is  induced by  rows ${\mathcal{A}}$ of $\bm{Y}$.  
\end{defi}

We remark that the general  separability condition is reasonable for many learning tasks, i.e.,  any learning task can be solved by the  matrix factorization model or  the latent variable model possessing  general separability \cite{zhou2014divide}.  

\subsection{Divide-and-conquer anchoring scheme for the general minimum conical hull problem}\label{sec:DCA}

\begin{figure}
\centering
{\begin{tikzpicture}[scale=0.6, transform shape]
\definecolor{brightcerulean}{rgb}{0.11, 0.67, 0.84}
\definecolor{bittersweet}{rgb}{1.0, 0.44, 0.37}
\definecolor{arsenic}{rgb}{0.23, 0.27, 0.29}
\definecolor{aquamarine}{rgb}{0.5, 1.0, 0.83}
\definecolor{antiquefuchsia}{rgb}{0.57, 0.36, 0.51}
\definecolor{violet(web)}{rgb}{0.93, 0.51, 0.93}
\shade[top color=red!40!white,opacity=0.55] (1.59,-0.63) --++(10+180:5) --+(80:5);
\draw[red!40,thick] (1.59,-.63) arc (10:80:5);
\draw[red!40,thick] (1.59,-.63) arc (80:10:-5);
\draw[red!40,thick] (1.59,-.63) -- ++(10+180:5) -- +(80:5);

\draw[bittersweet,fill=bittersweet] (-3.9,0.9) circle (.3ex);
\draw[antiquefuchsia, dashed, very thin] (-3.9,0.9)-- (-3.9,-1.5);
\draw[bittersweet!100,fill=bittersweet!100] (-3.9,-1.5) circle (.3ex); 

\draw[bittersweet,fill=bittersweet] (-2.55,2.9) circle (.3ex);
\draw[antiquefuchsia, dashed, very thin] (-2.55,2.9)-- (-2.55,-1.5);
\draw[bittersweet!100,fill=bittersweet!100] (-2.55,-1.5) circle (.3ex); 

\draw[bittersweet,fill=bittersweet] (1.7,0.2) circle (.3ex);
\draw[antiquefuchsia, dashed, very thin] (1.7,0.2)-- (1.7,-1.5);
\draw[bittersweet!100,fill=bittersweet!100] (1.7,-1.5) circle (.3ex); 

\draw[bittersweet,fill=bittersweet] (1.4,-0.8) circle (.3ex);
\draw[antiquefuchsia, dashed, very thin] (1.4,-0.8)-- (1.4,-1.5);
\draw[bittersweet!100,fill=bittersweet!100] (1.4,-1.5) circle (.3ex); 

\shade[top color=blue!40!white,opacity=0.75] (1,1) -- ++(30+180:5) -- +(60:5);
\draw[blue!40,thick] (1,1) arc (30:60:5);
\draw[blue!40,thick] (1,1) arc (60:30:-5);
\draw[blue!40,thick] (1,1) -- ++(30+180:5) -- +(60:5);
\draw[ultra thick,arsenic,->](-5,-1.5)--(2,-1.5) ;
\draw[blue!100,fill=blue!100] (0.2,0.7) circle (.3ex);
\draw[antiquefuchsia, dashed, very thin] (0.2,0.7)-- (0.2,-1.5);
\draw[blue!100,fill=blue!100] (0.2,-1.5) circle (.3ex);

\draw[blue!100,fill=blue!100] (-2.5,-0.7) circle (.3ex);
\draw[antiquefuchsia, dashed, very thin] (-2.5,-0.7)-- (-2.5,-1.5);
\draw[blue!100,fill=blue!100] (-2.5,-1.5) circle (.3ex);

\draw[blue!100,fill=blue!100] (-1.2,1.4) circle (.3ex);
\draw[antiquefuchsia, dashed, very thin] (-1.2,1.4)-- (-1.2,-1.5);
\draw[blue!100,fill=blue!100] (-1.2,-1.5) circle (.3ex);

\draw[blue!100,fill=blue!100] (-1,1.6) circle (.3ex);
\draw[antiquefuchsia, dashed, very thin] (-1.,1.6)-- (-1,-1.5);
\draw[blue!100,fill=blue!100] (-1,-1.5) circle (.3ex);
\draw[blue!100,fill=blue!100] (-1.95,0.9) circle (.3ex);
\draw[antiquefuchsia, dashed, very thin] (-1.95,0.9)-- (-1.95,-1.5);
\draw[blue!100,fill=blue!100] (-1.95,-1.5) circle (.3ex);

\draw[brightcerulean!100,fill=brightcerulean!100] (0.8,0.9) circle (.5ex);
\draw[antiquefuchsia, dashed, very thin] (0.8,0.9)-- (0.8,-1.5);
\draw[brightcerulean!100,fill=brightcerulean!100] (0.8,-1.5) circle (.5ex);
\definecolor{pistachio}{rgb}{0.58, 0.77, 0.45}
\draw[pistachio!100,fill=pistachio!200] (1,-0.7) circle (.5ex);
\draw[antiquefuchsia, dashed, very thin] (1,-0.7)-- (1,-1.5);
\draw[pistachio!100,fill=pistachio!200] (1,-1.5) circle (.5ex);

\definecolor{buff}{rgb}{0.94, 0.86, 0.51}
\draw[bittersweet!100,fill=bittersweet!200] (-2.2,2.4) circle (.3ex);
\draw[antiquefuchsia, dashed, very thin] (-2.2,2.4)-- (-2.2,-1.5);
\draw[bittersweet!100,fill=bittersweet!200] (-2.2,-1.5) circle (.3ex);
\node[text width=3cm] at (3.5,-1.5) 
    {1D hyperplane $\bm{B}_t$};
   \draw[black!30, dashed, very thin](1.5,0.6) rectangle +(2.5,2);
     \draw (3, 2.2) node [below left]{$\bm{X}=\{$};
     \draw[brightcerulean!100,fill=brightcerulean!100] (3,1.9) circle (.3ex);
      \draw[blue!100,fill=blue!100] (3.3,1.9) circle (.3ex);
      \draw (3.8, 2.2) node [below left]{$\}$};
    
   \draw (3.5, 2.7) node [below left]{$\bm{X}(j^*;:)=\{$};
     \draw[brightcerulean!100,fill=brightcerulean!100] (3.5,2.4) circle (.3ex);
      \draw (3.9, 2.7) node [below left]{$\}$};  
      
        \draw (3.6, 1.6) node [below left]{$\bm{Y}(\mathcal{A}_{t};:)=\{$};
     \draw[pistachio!100,fill=pistachio!200] (3.6,1.3) circle (.3ex);
      \draw (4.0, 1.6) node [below left]{$\}$};  
    
    \draw (3, 1.2) node [below left]{$\bm{Y}=\{$};
    \draw[bittersweet!100,fill=bittersweet!100] (3,.9) circle (.3ex);
     \draw[pistachio!100,fill=pistachio!200] (3.3,.9) circle (.3ex);
      \draw (3.8, 1.2) node [below left]{$\}$};  
\end{tikzpicture}}  
\caption{\small{The projection on a one-dimensional hyperplane. we show  how to obtain the anchor $\mathcal{A}_t$ as the solution in Eqn.~(\ref{eqn:MCH_1D}).  The dataset $\bm{Y}$ is composed of red and green nodes. The dataset $\bm{X}$ is a cone that is composed of blue and light blue nodes.   $\bm{B}_t$ represents a random one-dimensional  projection hyperplane. On hyperplane $\bm{B}_t$, the light blue node  $\bm{X}(j^*;:)$ refers  to the result of $\max_{j\in[n_{{\tiny{X}}}]}\bm{X}_t(j,:)$ in Eqn.~(\ref{eqn:MCH_1D}).   The green node $\bm{Y}(\mathcal{A}_t;:)$ corresponds to the solution of Eqn.~(\ref{eqn:MCH_1D}), where $\mathcal{A}_t$ refers to the anchor at the $t$-th random projection. Geometrically,  on the hyperplane $\bm{B}_t$, the green node is the nearest node among all nodes in $\bm{Y}$ relative to the light blue node and has a larger magnitude. }}
\label{fig:MCH}
\end{figure}
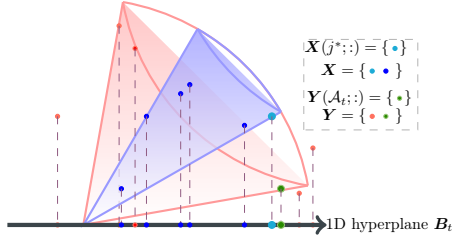%

Efficiently locating the  anchor set  ${\mathcal{A}}$ of a general minimum conical hull problem is challenging. To resolve this issue, the divide-and-conquer anchoring learning (DCA) scheme is proposed \cite{zhou2013divide,zhou2014divide}. This scheme adopts the following strategy. The original minimum conical hull problem is first reduced to a small number of  subproblems by projecting the original cone into low-dimensional hyperplanes. Then these subproblems on low-dimensional hyperplanes are then tackled in parallel. The anchor set of the original problem can be obtained by combining the anchor sets of the subproblems because 
the geometric information of  the original conical hull is partially preserved after projection.  Moreover, the efficiency of DCA schemes can be guaranteed because anchors in these subproblems can be effectively determined in parallel and the total number of subproblems is  modest. 

DCA is composed of two steps, i.e., the divide step and the conquer step.   Following the notations of Definition \ref{def:minimum_c_hull}, given  two sets of points (rows) with $\bm{X}\in\mathbb{R}^{n_{\scalebox{.49}{X}}\times m}$ and $\bm{Y}\in\mathbb{R}^{n_{\scalebox{.49}{Y}}\times m}$, the goal of DCA is to output an anchor set ${\mathcal{A}}$ such that    $\bm{X}(i;:)\in cone(\bm{Y}_{\mathcal{A}})$ with ${\mathcal{A}}\subset [n_{\scalebox{.49}{Y}}]$ and $|{\mathcal{A}}|=k$, $\forall i\in[n_{\scalebox{.49}{X}}]$. 

\underline{Divide step.} 
A set of  projection matrices $\{\bm{B}_t\}_{t=1}^p$ with $\bm{B}_t\in\mathbb{R}^{m\times d}$ is sampled from a random ensemble $\mathbb{M}$, e.g., Gaussian random matrix ensemble, the ensemble composed of standard unit vectors $\bm{e}_i$  or real data vectors, and various sparse random matrix ensembles  \cite{zhou2014divide}.    The general  minimum conical hull problem for the $t$-th  subproblem is to find an anchor set ${\mathcal{A}}_t$ for the projected matrices $\bm{Y}_t:=\bm{Y}\bm{B}_t\in\mathbb{R}^{n_{\scalebox{.49}{Y}}\times d}$ and $\bm{X}_t:= \bm{X}\bm{B}_t\in\mathbb{R}^{n_{\scalebox{.49}{X}}\times d}$: 
\begin{equation}\label{eqn:DCA_div}
\text{MCH}(\bm{Y}_t, \bm{X}_t):=	\arg\min \{|{\mathcal{A}}_t|: {cone}({\bm{Y}}_{{\mathcal{A}}_t})\supset {cone}(\bm{X}_t)\}~,
\end{equation}
where 
${\bm{Y}}_{{\mathcal{A}}_t}$ is the submatrix of $\bm{Y}$ whose rows are indexed by ${\mathcal{A}}_t\subset [n_{\scalebox{.49}{Y}}]$. 

\underline{Conquer step.} 
This step yields the anchor set ${\mathcal{A}}$ by employing the following selection rule to manipulate the collected $\{{{\mathcal{A}}}_t\}_{t=1}^p$. First, we compute, $\forall i\in[n_{\scalebox{.49}{Y}}]$,
\begin{equation}\label{eqn:conquer}
\hat{g}_i:= \frac{1}{p}\sum_{t=1}^p\mathbbm{1}_{{\mathcal{A}}_t}(\bm{Y}(i;:))~,
\end{equation}
where $\bm{Y}(i;:)$ is the $i$-th row of $\bm{Y}$, and $\mathbbm{1}_{{{\mathcal{A}}}_t}(\bm{Y}(i;:))$ is the indicator function that outputs `1' when the index $i$ is in ${{\mathcal{A}}}_t$, and zero otherwise. The anchor set  ${\mathcal{A}}$ of size $k$ is constructed by selecting $k$ indexes with the largest $\hat{g}_i$.
 
It has been proved, with high probability that, solving $p =\mathcal{ O}(k \log k)$ subproblems are sufficient to find all anchors in $\mathcal{A}$, where $k$ refers to the number of anchors \cite{zhou2014divide}. The total runtime complexity of DCA is     $\tilde{\mathcal{O}}(\max\{\text{poly}(n_{\scalebox{.49}{X}}), \text{poly}(n_{\scalebox{.49}{Y}})\}) $ when  parallel processing is allowed.

\subsection{Reformulation as a sampling problem}\label{subsec:MCH_1D}

Interestingly,  each subproblem in Eqn.~(\ref{eqn:DCA_div}) can be  reduced to a sampling problem when $d=1$. This observation is one of the crucial components that make our algorithm exponentially faster than the conventional DCA schemes \cite{zhou2014divide}. 

Fix $d=1$. Following the result of \cite{yu2016scalable}, Eqn.~(\ref{eqn:DCA_div}) becomes
\begin{equation}\label{eqn:MCH_1D}
 \text{MCH}(\bm{Y}_t, \bm{X}_t):= 	\mathcal{A}_t = \left\{\arg\min_{i\in[n_{\scalebox{.49}{Y}}]}(\bm{Y}_t(i,:)-\max_{j\in[n_{\scalebox{.49}{X}}]}\bm{X}_t(j,:))_+\right\}~,
\end{equation}
where $\bm{Y}_t\in\mathbb{R}^{n_{\scalebox{.49}{Y}}\times 1}$, $\bm{X}_t\in\mathbb{R}^{n_{\scalebox{.49}{X}}\times 1}$, $(x)_+ =x$ if $x\geq 0$ and $\infty$ otherwise. We give an intuitive explanation of Eqn.~(\ref{eqn:MCH_1D}) in Figure (\ref{fig:MCH}). 

Note that Eqn.~(\ref{eqn:MCH_1D}) can then be written as
\begin{equation}\label{eqn:QDCA_MCH_1D}
 \text{MCH}(\bm{Y}_t, \bm{X}_t):= 	\mathcal{A}_t = \left\{ \arg\min_{i\in[n_{\scalebox{.49}{Y}}]}\left(\mathcal{P}_{\bm{Y}_t}(i)-\xi_t\max_{j\in[n_{\scalebox{.49}{X}}]}(\mathcal{P}_{\bm{X}_t}(j))\right)_+\right\}~, \forall t\in[p]~,
\end{equation}
where  $\mathcal{P}_{\bm{X}_t}$ and $\mathcal{P}_{\bm{Y}_t}$ refer to the distributions of  $\bm{X}_t$ and $\bm{Y}_t$, and $\xi_t= \|\bm{X}_t\|/\|\bm{Y}_t\|$ is a constant at $t$-th random projection to rescale the value $\max_{j\in[n_{\scalebox{.49}{X}}]}(\mathcal{P}_{\bm{X}_t}(j))$ . We will explain how to  efficiently approximate $\xi_t$  in Section \ref{subsec:DS}.

\section{Main Algorithm}\label{Sec:main_algorithm}
Our algorithm generates two distributions  $\mathcal{P}_{\bm{\hat{X}}_t}$ and $\mathcal{P}_{\bm{\hat{Y}}_t}$ to approximate the targeted distributions $\mathcal{P}_{\bm{X}_t}$ and $\mathcal{P}_{\bm{Y}_t}$ as defined in Eqn.~(\ref{eqn:QDCA_MCH_1D}) for any $t\in[p]$. The core of our algorithm consists of the following major steps. In the preprocessing step, we reconstruct two matrices, $\tilde{\bm{X}}$ and $\tilde{\bm{Y}}$, to approximate  the original matrices $\bm{X}$ and $\bm{Y}$ so that the $t$-th  projected subproblem can be replaced with  $\tilde{\bm{X}}_t=\tilde{\bm{X}}\bm{B}_t $ and $\bm{\tilde{Y}}_t	=\tilde{\bm{Y}}\bm{B}_t$ with little disturbance. {The main tool for this approximation is the subsampling method  with the support  of the square-length sampling operations \cite{tang2018quantum}. This step also appears in \cite{chia2018quantum,gilyen2018quantum,tang2018quantum}. All  subproblems are processed in parallel, following the divide-and-conquer principle. The divide step employs two sampling subroutines that allow us to efficiently generate two distributions $\mathcal{P}_{\bm{\hat{X}}_t}$ and $\mathcal{P}_{\bm{\hat{Y}}_t}$ to approximate $\mathcal{P}_{\bm{\tilde{X}}_t}$ and $\mathcal{P}_{\tilde{\bm{Y}}_t}$ (equivalently,  $\mathcal{P}_{\bm{{X}}_t}$ and $\mathcal{P}_{{\bm{Y}}_t}$).  We then propose the general heuristic post-selection rule to identify the target index $\mathcal{A}_t$ by substituting  $(\mathcal{P}_{\bm{Y}_t}(i)-\xi_t\max_{j\in[n_{\scalebox{.49}{X}}]}(\mathcal{P}_{\bm{X}_t}(j))_+$ in Eqn~(\ref{eqn:QDCA_MCH_1D}) with $(\mathcal{P}_{\bm{\hat{Y}}_t}(i)-\xi_t\max_{j\in[n_{\scalebox{.49}{X}}]}(\mathcal{P}_{\bm{\hat{X}}_t}(j))_+$. Lastly, we employ the selection rule in Eqn.~(\ref{eqn:conquer}) to form the anchor set $\mathcal{A}$ for the original minimum conical hull problem. 


Before elaborating the details of the main algorithm, we first emphasize the innovations of this work, i.e.,  the reformulation of the general conical hull problem as a sampling problem and the general heuristic post-selection rule. The sampling version of the general conical hull problem is the precondition to introduce advanced sampling techniques to reduce the computational complexity.  The general heuristic post-selection rule is the central component to guarantee  that the solution of the general conical hull problem can be obtained in sublinear runtime. In particular, the intrinsic mechanism of the general conical hull problem enables us to employ the general heuristic post-selection rule to query a specific element from the output in polylogarithmic runtime. 

In this section, we introduce the length-square sampling operations   in Subsection \ref{subsec:samp_assp}. The  implementation of our algorithm is shown in Subsection \ref{subsec:DS}. The computation analysis is given in Subsection \ref{subsec:comp}. We give  the proof of Theorems   \ref{thm:post_sele} and \ref{thm:comp}  in supplementary material \ref{sec:SM_gene_post} and \ref{sec:SM_compl},  respectively.
  
  \subsection{Length-square sampling operations}\label{subsec:samp_assp}
  The promising performance of various sampling algorithms for machine learning tasks is guaranteed when the given dataset supports length-square sampling operations   \cite{chia2018quantum,gilyen2018quantum,tang2018quantum}. We first give the definition of the access cost to facilitate the description of such   sampling operations, 
\begin{defi}[Access cost]\label{def:acces_cost}
	Given a matrix $\bm{W}\in \mathbb{R}^{n\times m}$, we denote that the cost of querying an entry $\bm{W}(i,j)$ or querying the Frobenius norm $\|\bm{W}\|_F$ is $Q(\bm{W})$,  the cost of querying the $\ell_2$  norm of $\|\bm{W}(i,:)\|$ is $N(\bm{W})$, and the  cost of sampling a row index $i\in[n]$ of $\bm{H}$ with the probability $\mathcal{P}_{\bm{W}}(i)=\|\bm{W}(i,:) \|^2/\|\bm{W}\|_F^2$ or sampling an index $j\in[m]$ with the probability $\mathcal{P}_{\bm{W}(i,:)}(j)$  is   $S(\bm{W})$. We denote the overall access cost of $\bm{W}$ as $L(\bm{W}) := S(\bm{W}) + Q(\bm{W}) + N(\bm{W})$.
\end{defi} 
We use an example to address the difference between  query access and sampling access. Given a vector $\bm{v}\in\mathbb{R}^n$, the problem is to find a hidden large entry $\bm{v}(i)$. It takes  $\Omega(n)$ cost to find  $\bm{v}(i)$ with just query access, while the computation cost is $\mathcal{O}(1)$ with query and sample access.

  The  length-square sampling  operations, \textbf{as so-called $\ell_2$ norm sampling  operations},  are defined as:
\begin{prop}[$\ell_2$ norm sampling operation]\label{assmp:1}
Given an input matrix $\bm{H}\in \mathbb{R}^{n\times m}$ with $s$ non-zero entries, there exists a data structure storing $\bm{H}$ in space $\mathcal{O}(s\log^2 m)$,  with  the following properties: 
\begin{enumerate}
\item The query cost $Q(\bm{H})$ is at most  $\mathcal{O}(\log (nm))$; 
\item The query cost $N(\bm{H})$ is at most  $\mathcal{O}(\log ^2(n))$; 
\item The sampling cost $S(\bm{H})$ is at most $\mathcal{O}(\log^2 (nm))$. 
\end{enumerate}
The overall cost of accessing $\bm{H}$ is therefore $L(\bm{H})=\mathcal{O}(poly(\log(mn)))$.
\end{prop}


\textbf{Remark.} The $\ell_2$ norm   sampling operation can be efficiently fulfilled  if  the input data are stored in a low-overhead data structure, e.g., the binary tree structure (BNS)  \cite{kerenidis2017quantum} (more details about BNS are  given in supplementary material \ref{sec:SM_BNS}). 

\subsection{The implementation of the algorithm}\label{subsec:DS}
Our algorithm consists of three  steps, the preprocessing step,  the divide step, and the conquer step. The first step prepares an efficient description for $\bm{X}$ and $\bm{Y}$. The second step locates anchors $\{\mathcal{A}_t\}_{t=1}^p$ by solving $p$ subproblems in parallel. The last step obtains the anchor set $\mathcal{A}$.   

\textit{\underline{Preprocessing step.}}
The preprocessing step aims to efficiently construct $\bm{\tilde{X}}$ and $\bm{\tilde{Y}}$ such that the matrix norm  $\|\bm{\tilde{X}}- \bm{X}\|_2$ and $\|\bm{\tilde{Y}}- \bm{Y}\|_2$ are small. This step employs the subsampling method \cite{tang2018quantum} to construct an approximated left singular matrix $\tilde{\bm{V}}_{\scalebox{.49}{H}}$ of $\bm{H}$, where $\bm{H}\in\mathbb{R}^{n_{\scalebox{.49}{H}}\times m_{\scalebox{.49}{H}}}$ can be either $\bm{X}$ or $\bm{Y}$, so that $\bm{\tilde{H}}= \tilde{\bm{V}}_{\scalebox{.49}{H}}\tilde{{\bm{V}}_{\scalebox{.49}{H}}}^{\top}\bm{H}$. If no confusion arises, the subscript $\bm{H}$ can be  disregarded. 

We summarize the subsampling method in Algorithm \ref{alg:subsamp} to detail   the acquisition of   $\tilde{\bm{V}}$. We build the matrix $\bm{R}\in\mathbb{R}^{n\times s}$ by sampling $s$ columns from $\bm{H}$ and then build the  matrix $\bm{C}\in\mathbb{R}^{s \times s}$ by sampling $s$ rows from $\bm{R}$. 
After obtaining $\bm{R}$ and $\bm{C}$, we implicitly  define the  approximated left singular matrix $\tilde{\bm{V}}\in\mathbb{R}^{n\times k}$ as 
\begin{equation}\label{eqn:apprx_V}
	\tilde{{\bm{V}}}(:,i)\equiv \tilde{\bm{v}}^{(i)}:=\frac{\bm{R}\bm{\omega}^{(i)}}{\sigma^{(i)}}~,  
\end{equation} 
where $\{\sigma^{(i)}\}_{i=1}^{k}$ and $\{\bm{\omega}^{(i)}\}_{i=1}^{k}$ refer to $k$ singular values and right singular vectors of $\bm{C}$\footnote{\scriptsize{The `implicitly define $\bm{\tilde{V}}$' means that  only the index array of $\bm{R}$ and $\bm{C}$, and the SVD result of $\bm{C}$ are required to be stored in the memory.}}. 


\begin{algorithm}
 \KwData{$\bm{H}\in\mathbb{R}^{n\times m}$ with supporting $\ell_2$ norm sampling operations,  parameters $s$, $\epsilon$, $\kappa$.  }
 \KwResult{The singular value decomposition of $\bm{C}$.}
 Independently sample $s$ columns indices $[i_s]$ according to the probability distribution $\mathcal{P}_{\bm{H}}$\;
 Set $\bm{R}\in\mathbb{R}^{n\times s}$ as the matrix formed by $\bm{H}(:,i_t)/\sqrt{s\mathcal{P}_{\bm{H}(:,i_t)}}$ with $i_t\in[i_s]$\;
 Sample a column index $t$ with $t\in[s]$ uniformly and then sample a row index $j\in[n]$ distributed as $\mathcal{P}_{\bm{R}(j,t)}$. Sample a total number of $s$ row indexes $[j_s]$ in this way\; 
 Let $\bm{C}\in\mathbb{R}^{s\times s }$ be a matrix whose $t$-th row is $\bm{C}(t,:)=\bm{R}(j_t,:)/\sqrt{s\mathcal{P}_{\bm{R}(j_t,:)}}$.
 
 Apply SVD to obtain right singular vector $\{\bm{\omega}^{(i)}\}_{i=1}^k$  and singular values $\{\sigma^{(i)}\}_{i=1}^k$ of $\bm{C}$ \;
 
 Output the decomposition results  $\{\sigma^{(i)}, \bm{\omega}^{(i)}\}_{i=1}^k$ of $\bm{C}$.
\caption{Subsampling method \cite{tang2018quantum}}
 \label{alg:subsamp}
\end{algorithm}

\textit{\underline{Divide step.}} The obtained  approximated left singular matrices $\tilde{\bm{V}}_{\scalebox{.49}{X}}\in\mathbb{R}^{n_{\scalebox{.49}{X}}\times k_{\scalebox{.49}{X}}}$ and $\tilde{\bm{V}}_{\scalebox{.49}{Y}}\in \mathbb{R}^{n_{\scalebox{.49}{Y}}\times k_{\scalebox{.49}{Y}}}$ enable us to employ advanced sampling techniques to locate potential anchors  $\{\mathcal{A}_t\}_{t=1}^p$. Here we only focus on  locating the anchor $\mathcal{A}_t$ for the $t$-th subproblem, since each subproblem is independent and can be solved in the same way. The divide step employs Eqn.~(\ref{eqn:QDCA_MCH_1D}) to locate $\mathcal{A}_t$. In particular, we first prepare two distributions $\mathcal{P}_{\bm{\hat{X}}_t}$ and $\mathcal{P}_{\bm{\hat{Y}}_t}$ to approximate $\mathcal{P}_{\bm{X}_t}$ and $\mathcal{P}_{\bm{Y}_t}$, and then sample these two distributions to locate $\mathcal{A}_t$.  

The preparation of two distributions $\mathcal{P}_{\bm{\hat{X}}_t}$ and $\mathcal{P}_{\bm{\hat{Y}}_t}$ is achieved by exploiting  two  sampling subroutines, the inner product subroutine and the thin matrix-vector multiplication subroutine \cite{tang2018quantum}. We detail these two subroutines in supplementary material \ref{sec:SM_samp}. Recall that the two approximated matrices at the $t$-th subproblem are 
${\bm{\tilde{X}}_t}	= \tilde{{\bm{V}}}_{\scalebox{.49}{X}}(\tilde{{\bm{V}}}_{\scalebox{.49}{X}}^{\top}\bm{X}_t)$ and ${\bm{\tilde{Y}}_t} = \tilde{\bm{V}}_{\scalebox{.49}{Y}}(\tilde{{\bm{V}}}_{\scalebox{.49}{Y}}^{\top}\bm{Y}_t).$
Denote $\bm{\tilde{q}}_{\scalebox{.49}{X},t}\equiv\tilde{{\bm{V}}}_{\scalebox{.49}{X}}^{\top}\bm{X}_t\in\mathbb{R}^{k_{\scalebox{.49}{X}}\times 1} $ and $\bm{\tilde{q}}_{\scalebox{.49}{Y},t}\equiv\tilde{{\bm{V}}}_{\scalebox{.49}{Y}}^{\top}\bm{Y}_t \in\mathbb{R}^{k_{\scalebox{.49}{Y}}\times 1} $. 
Instead of directly computing $\bm{\tilde{q}}_{\scalebox{.49}{X},t}$ and $\bm{\tilde{q}}_{\scalebox{.49}{Y},t}$,   we construct their approximated vectors $\hat{\bm{q}}_{\scalebox{.49}{X},t}$ and $\hat{\bm{q}}_{\scalebox{.49}{Y},t}$ using the inner product subroutine
to ensure the low computational  cost, followed by  the thin matrix-vector multiplication subroutine to  prepare   probability distributions $\mathcal{P}_{\bm{\hat{X}}_t}$ and  $\mathcal{P}_{\bm{\hat{Y}}_t}$, where $\bm{\hat{X}}_t \equiv  \tilde{{\bm{V}}}_{\scalebox{.49}{X}} \hat{\bm{q}}_{\scalebox{.49}{X},t}$ and $\bm{\hat{Y}}_t \equiv  \tilde{{\bm{V}}}_{\scalebox{.49}{Y}} \hat{\bm{q}}_{\scalebox{.49}{Y},t}$.
The closeness between  $\mathcal{P}_{\bm{\hat{X}}_t}$  (resp. $\mathcal{P}_{\bm{\hat{Y}}_t}$) and $\mathcal{P}_{\bm{{X}}_t}$ (resp. $\mathcal{P}_{\bm{{Y}}_t}$) is controlled by the number of samplings $s$, as analyzed in Section \ref{sec:corrts}. 
  
The rescale parameter $\xi_t$ defined in Eqn.~(\ref{eqn:QDCA_MCH_1D}) can be efficiently approximated by employing the inner product subroutine. Recall that $\xi_t = \|\bm{Y}_t\|/\|\bm{X}_t\|$. We can approximate $\xi_t$ by  $\hat{\xi}_t=\|\bm{\hat{Y}}_t\|/\|\bm{\hat{X}}_t\|$. Alternatively, an efficient method of approximating $\|\bm{\hat{Y}}_t\|$ and $\|\bm{\hat{X}}_t\|$ is sufficient to acquire $\xi_t$. Let $\bm{H}$ be the general setting that can either be $\bm{X}$ or $\bm{Y}$. The $\ell_2$ norm of $\|\bm{\hat{H}}_t\|$ can be efficiently estimated by using the inner product subroutine. Intuitively, the $\ell_2$ norm of $\|\hat{\bm{H}}_t\|$ can be expressed  by the inner product of $\hat{\bm{H}}_t$, i..e,  $\|\hat{\bm{H}}_t\|^2= {\hat{\bm{H}}_t^{\top}\hat{\bm{H}}_t}$. Recall that $\hat{\bm{H}}_t$ has an explicit representation $\hat{\bm{H}}_t=\tilde{\bm{V}}_{\scalebox{.49}{H}}\hat{\bm{q}}_{\scalebox{.49}{H,t}}$, and the efficient access cost for $\tilde{\bm{V}}_{\scalebox{.49}{H}}$ and $\tilde{\bm{q}}_{\scalebox{.49}{H,t}}$ enables us to use the inner product subroutine to efficiently obtain $\|\bm{\tilde{H}}_t\|$. 

Given $\mathcal{P}_{\bm{\hat{X}}_t}$, $\mathcal{P}_{\bm{\hat{Y}}_t}$, and $\hat{\xi_t}$, we propose the general heuristic post-selection method to determine $\mathcal{A}_t$.   Following the sampling version of the general minimum conical hull problem defined in Eqn.~(\ref{eqn:QDCA_MCH_1D}), we first sample the   distribution $\mathcal{P}_{\bm{\hat{X}}_t}$ with $N_{\scalebox{.49}{X}}$ times to obtain a value $C_{\bm{\hat{X}}_t}^*$ such that $ C_{\bm{\hat{X}}_t}^*=\max_{j\in[n_{\scalebox{.49}{X}}]}(\mathcal{P}_{\bm{\hat{X}}_t}(j))$.   We next sample the distribution $\mathcal{P}_{\bm{\hat{Y}}_t}$ with $N_{\scalebox{.49}{Y}}$ times to find an index $\hat{\mathcal{A}}_t$ approximating ${\mathcal{A}}_t= \arg\min_{i\in[n_{\scalebox{.49}{Y}}]}(\mathcal{P}_{\bm{\hat{Y}}_t}(i)- \hat{\xi}_tC_{\bm{\hat{X}}_t}^*)_+$.  The following theorem quantifies the required  number of  samplings to guarantee $\hat{\mathcal{A}}_t=\mathcal{A}_t$, where $\mathcal{A}_t$ is defined in Eqn.~(\ref{eqn:QDCA_MCH_1D}). 
\begin{thm}[General heuristic post-selection (Informal)]\label{thm:post_sele}
	Assume that  $\mathcal{P}_{\bm{\hat{X}}_t}$ and $\mathcal{P}_{\bm{\hat{Y}}_t}$ are multinomial distributions. Denote that $C_{\bm{\hat{X}}_t}^*\geq \varepsilon_T$.  If $|\mathcal{P}_{\bm{\hat{X}}_t}(i) - \mathcal{P}_{\bm{\hat{X}}_t}(j)|> \varepsilon$ for $C_{\bm{\hat{X}}_t}^*=\mathcal{P}_{\bm{\hat{X}}_t}(i)$ and $\forall j\neq i$, and   $|\mathcal{P}_{\bm{\hat{Y}}_t}(\mathcal{A}_t) - \mathcal{P}_{\bm{\hat{Y}}_t}(j)|> \varepsilon$ for $\forall j\neq \mathcal{A}_t$ and a  small constant $\varepsilon$, then for any $\delta > 0$ with a probability at least $1 -\delta$, we have $\hat{\mathcal{A}}_t=\mathcal{A}_t$ with  $N_{\scalebox{.49}{X}}, N_{\scalebox{.49}{Y}}\sim  \mathcal{O}(\kappa_{\scalebox{.49}{X}},  polylog(\max\{n_{\scalebox{.49}{X}},n_{\scalebox{.49}{Y}}\}))$.\end{thm}


\textit{\underline{Conquer step.} }
After the divide step, a set of potential anchors $\{\mathcal{A}_t\}_{t=1}^p$ with $p=\mathcal{O}(k\log(k))$, where $k$ refers to the number of anchors with $k\sim \min\{\log(n_{\scalebox{.49}{X}}), \log(n_{\scalebox{.49}{Y}})\}$,  is collected \cite{du2018quantum,yu2016scalable,zhou2014divide}. We adopt  the selection rule defined in Eqn.~(\ref{eqn:conquer}) to determine the anchor set $\mathcal{A}$. This step can be  accomplished by various sorting algorithms with runtime $\mathcal{O}(poly(k\log(k)))$  \cite{zhou2014divide}. 

We summarize our algorithm as follows. 
\begin{algorithm}
 \KwData{Given $\bm{X}\in\mathbb{R}^{n_{\scalebox{.49}{X}}\times m}$, $\bm{Y}\in\mathbb{R}^{n_{\scalebox{.49}{Y}}\times m}$, and $\{\bm{B}_t\}_{t=1}^p$ with $\bm{B}_t\in\mathbb{R}^{m\times 1}$. Both $\bm{X}$, $\bm{Y}$, and $\bm{B}_t$ supports $\ell_2$ sampling operations. The number of anchors $k$. }
 \KwResult{Output the set of anchors $\mathcal{A}$ with $|\mathcal{A}|=k$.}
 \textit{\underline{Preprocessing step}}  \;
 Set $s_{\scalebox{.49}{X}}\leftarrow \frac{85^2k_{\scalebox{.49}{X}}^2\kappa_{\scalebox{.49}{X}}^4\ln(8n_{\scalebox{.49}{X}}/\eta)\|\bm{X}\|_F^2}{9\epsilon^2 }$ and $s_{\scalebox{.49}{Y}}\leftarrow \frac{85^2k_{\scalebox{.49}{Y}}^2\kappa_{\scalebox{.49}{Y}}^4\ln(8n_{\scalebox{.49}{Y}}/\eta)\|\bm{Y}\|_F^2}{9\epsilon^2 }$ (See Theorem  \ref{thm:sample_Alg})\;
 Sample $\bm{R}_{\scalebox{.49}{X}}\in\mathbb{R}^{n_{\scalebox{.49}{X}}\times s_{\scalebox{.49}{X}}}$ from $\bm{X}$ and sample $\bm{C}_{\scalebox{.49}{X}}\in\mathbb{R}^{s_{\scalebox{.49}{X}}\times s_{\scalebox{.49}{X}}}$ from $\bm{R}_{\scalebox{.49}{X}}$ using Algorithm \ref{alg:subsamp} \;
  Sample $\bm{R}_{\scalebox{.49}{Y}}\in\mathbb{R}^{n_{\scalebox{.49}{Y}}\times s_{\scalebox{.49}{Y}}}$ from $\bm{Y}$ and sample $\bm{C}_{\scalebox{.49}{Y}}\in\mathbb{R}^{s_{\scalebox{.49}{Y}}\times s_{\scalebox{.49}{Y}}}$ from $\bm{R}_{\scalebox{.49}{Y}}$  using Algorithm \ref{alg:subsamp} \;
  Implicitly define  $\tilde{\bm{V}}_{\scalebox{.49}{X}}$ and  $\tilde{\bm{V}}_{\scalebox{.49}{Y}}$ by employing Eqn.~(\ref{eqn:apprx_V}) \;

(\textit{\underline{ Divide Step}}) Set  $t=0$\;
\While{$t<p$ (Computing in  parallel.)}{
Estimate  $\tilde{\bm{q}}_{\scalebox{.49}{X}}^t= \tilde{\bm{V}}_{\scalebox{.49}{X}}^{\top}\bm{X}_t$ and $\tilde{\bm{q}}_{\scalebox{.49}{Y}}^t= \tilde{\bm{V}}_{\scalebox{.49}{Y}}^{\top}\bm{Y}_t$ by $\hat{\bm{q}}_{\scalebox{.49}{X}}^t$ and $\hat{\bm{q}}_{\scalebox{.49}{Y}}^t$  via inner product subroutine\;
Prepare  $\mathcal{P}_{\bm{\hat{X}}_t}$ ($\mathcal{P}_{\bm{\hat{Y}}_t}$) with $\bm{\hat{X}}_t=\tilde{\bm{V}}_{\scalebox{.49}{X}} \hat{\bm{q}}_{\scalebox{.49}{X}}^t$ ($\bm{\hat{Y}}_t=\tilde{\bm{V}}_{\scalebox{.49}{Y}} \hat{\bm{q}}_{\scalebox{.49}{Y}}^t$) via  matrix-vector multiplication subroutine\; 
Collect $\mathcal{\hat{A}}^t $ by sampling $\mathcal{P}_{\hat{\bm{X}}_t}$ and $\mathcal{P}_{\hat{\bm{Y}}_t}$  via the {general heuristic post-selection} method\; 
$t\leftarrow t+1$\;
}

(\textit{\underline{ Conquer Step}}) Output the anchor set $\mathcal{A}$ using the selection rule defined in  Eqn.~(\ref{eqn:conquer})\;
 \caption{A sublinear runtime  algorithm for the general minimum conical hull problem}
 \label{alg:meta}
\end{algorithm}

\subsection{Computation complexity of the algorithm}\label{subsec:comp}
The complexity of the proposed algorithm is dominated by the preprocessing step and the divide step. Specifically, the computational complexity is occupied by four elements:   finding the left singular matrix $\tilde{\bm{V}}$ (Line 5 in Algorithm \ref{alg:meta}),  estimating $\tilde{\bm{q}}_t$ by $\hat{\bm{q}}_t$ (Line 7 in Algorithm \ref{alg:meta}), preparing the approximated probability distribution $\mathcal{P}_{\hat{\bm{H}}}$ (Line 8 in Algorithm \ref{alg:meta}), and estimating the rescale factor $\hat{\xi}_t$ (Line 10 in Algorithm \ref{alg:meta}). We  evaluate the computation complexity of these four operations separately and obtain the following theorem 
\begin{thm}\label{thm:comp}
Given two datasets $\bm{X}\in\mathbb{R}^{n_{\scalebox{.49}{X}}\times m}$ and $\bm{Y}\in\mathbb{R}^{n_{\scalebox{.49}{Y}}\times m}$ that satisfy the general separability condition and  support the length-square sampling operations, the rank and condition number for $\bm{X}$ (resp. $\bm{Y}$) are denoted as $k_{\scalebox{.49}{X}}$ (resp. $k_{\scalebox{.49}{Y}}$) and $\kappa_{\scalebox{.49}{X}}$ (resp. $\kappa_{\scalebox{.49}{Y}}$). The tolerable error is denoted as $\epsilon$. 
 The  runtime complexity of the proposed algorithm for solving  the general minimum conical hull problem  is   
$$\max\left\{\tilde{\mathcal{O}}(\frac{85^6k_{\scalebox{.49}{X}}^6\kappa_{\scalebox{.49}{X}}^{12}\|\bm{X}\|_F^6}{9^3\epsilon^6 }), \tilde{\mathcal{O}}(\frac{85^6k_{\scalebox{.49}{Y,t}}^6\kappa_{\scalebox{.49}{Y}}^{12}\|\bm{Y}\|_F^6}{9^3\epsilon^6 }) \right\}~.$$ 	
\end{thm}

\section{Correctness  of The Algorithm}\label{sec:corrts}
  In this section, we present the correctness of our algorithm. We also briefly explain how the results are derived and validate our algorithm by  numerical simulations. We provide the detailed proofs of Theorem \ref{thm:main_corect} in the supplementary material \ref{sec:SM_correct}. 
  
  The correctness of our algorithm is determined by the closeness between  the approximated distribution and the analytic distribution. The closeness is evaluated by  the total variation distance \cite{tang2018quantum}, i.e., 	 Given two vectors $\mathbf{v}\in\mathbb{R}^{n}$ and $\mathbf{w}\in\mathbb{R}^{n}$, the total variation distance of two distributions $\mathcal{P}_{{\mathbf{v}}}$ and $\mathcal{P}_{{\mathbf{w}}}$ is  $\|\mathcal{P}_{{\mathbf{v}}}, \mathcal{P}_{{\mathbf{w}}}\|_{TV}:= \frac{1}{2}\sum_{i=1}^n|\mathcal{P}_{{\mathbf{v}}}(i) - \mathcal{P}_{{\mathbf{w}}}(i) |$.
 The following theorem states that  $\|\mathcal{P}_{\bm{\hat{X}}_t}-\mathcal{P}_{\bm{{X}}_t}\|_{TV}$ and $\|\mathcal{P}_{\bm{\hat{Y}}_t}- \mathcal{P}_{\bm{{Y}}_t}\|_{TV}$ are controlled by  the number of samplings $s$: 
   \begin{thm}\label{thm:main_corect}
	Given a matrix $\bm{H}\in\mathbb{R}^{n\times m}$ with  $\ell_2$ norm sampling operations, 
	let $\bm{R}\in\mathbb{R}^{n\times s}$, $\bm{C}\in\mathbb{R}^{s\times  s}$ be the sampled matrices from $\bm{H}$ following Algorithm \ref{alg:subsamp}. The distribution prepared by the proposed algorithm is denoted as $\mathcal{P}_{\bm{\hat{H}}_t}$. Set   $$s=\frac{85^2k^2\kappa^4\ln(8n/\eta)\|\bm{H}\|_F^2}{9\epsilon^2 },$$ with probability at least $(1-\eta)$, we always have $  \|\mathcal{P}_{\bm{\hat{H}}_t}-\mathcal{P}_{\bm{{H}}_t}\|_{TV}\leq \epsilon$.
\end{thm}

 We apply the proposed algorithm to accomplish the near separable nonnegative matrix factorization (SNMF) to validate the correctness of Theorem \ref{thm:main_corect} \cite{zhou2013divide,kumar2013fast}. SNMF, which has been extensive applied to hyperspectral imaging and text mining, can be treated as a special case of the general minimum conical hull problem \cite{zhou2014divide}.   Given a non-negative matrix $\bm{X}$, SNMF aims to find a decomposition such that $\bm{X} = \bm{F}\bm{X}(\mathcal{R},:)+\bm{N}$, where the basis matrix $\bm{X}(\mathcal{R},:)$ is composited by $r$ rows from $\bm{X}$, $\bm{F}$ is the non-negative encoding matrix, and $\bm{N}$ is a noise matrix   \cite{donoho2004does}. 
 
 The synthetic data matrix used in the experiments is generated in the form of $\bm{Y}\equiv \bm{X} = \bm{F}\bm{X}_{\mathcal{A}} +\bm{N}_G$, where $\bm{F} = [\bm{I}_k; \bm{\mathcal{U}}]$, the entries of noise $\bm{N}_G$ are generated by sampling Gaussian distribution $\mathcal{N}(0,\mu)$ with noise level $\mu$,   the entries of  $\bm{X}_{\mathcal{A}} $ and $\bm{\mathcal{U}}$  are generated by i.i.d. uniform distribution in $[0, 1]$ at first and then their rows  are normalized to have unit $\ell_1$ norm. Analogous to DCA, we use the recovery rate as the metric to evaluate the precision of anchor recovery.  The anchor index recovery rate $\rho$ is defined as $\rho=|\mathcal{A}\cap \hat{\mathcal{A}}|/|\mathcal{A}|$, where  $\hat{\mathcal{A}}$ refers to the anchor set obtained by our algorithm or DCA. 

We set the dimensions of $\bm{X}$ as $500\times 500$ and set $k$ as 10. We generate four datasets with  different noise levels, which are $\mu=0, 0.1, 0.5, 2$.  The number of subproblems is set as $p=100$. We give nine different settings for  the number of sampling $s$ for our algorithm, ranging from $500$ to $8000$. We compare our algorithm with DCA to determine the anchor set $\mathcal{A}$. The simulation results are shown in Figure \ref{fig:SNMF}. For the case of $\mu=0$, both our algorithm and DCA can accurately locate all anchors. With increased noise, the recovery rate  continuously decreases both for our algorithm and DCA, since the separability assumption is not preserved.  As shown in Figure \ref{fig:SNMF} (b), the reconstruction  error, which is evaluated by the Frobenius norm of $\|\hat{\bm{X}}-\bm{X}\|_F$,  continuously decreases with  increased $s$ for the noiseless case. In addition, the variance of the reconstructed error, illustrated by the shadow region,  continuously shrinks with increased $s$. For the high noise level case, the collapsed separability assumption arises that the reconstruction error is unrelated to $s$.  


 \begin{figure}
      \centering
     \begin{subfigure}[b]{0.41\textwidth}
          \centering
          \resizebox{\linewidth}{!}{\includegraphics[width=0.41\textwidth]{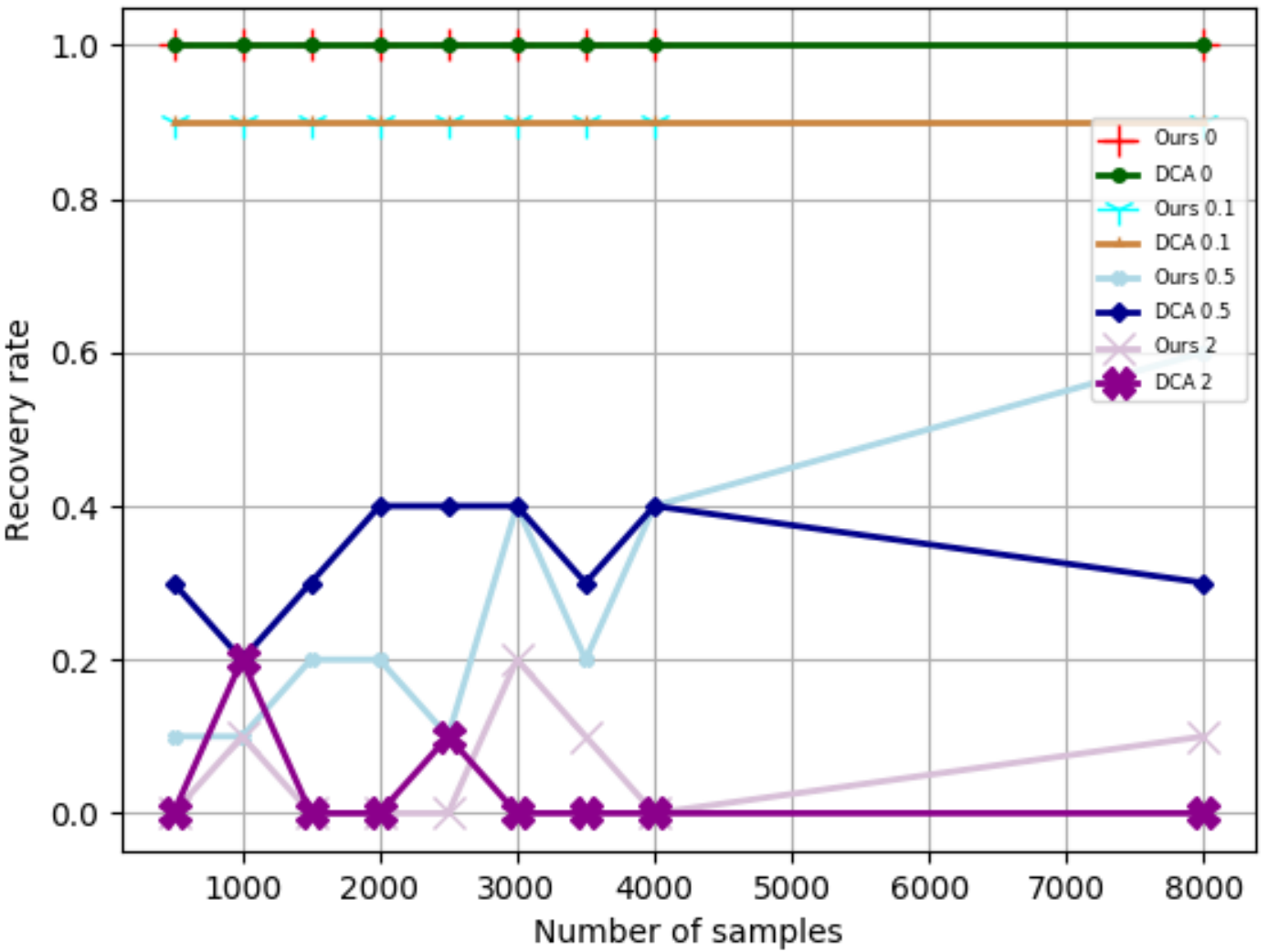}} 
          \caption{Recovery rate}
          \label{fig:A}
     \end{subfigure}
     \begin{subfigure}[b]{0.41\textwidth}
          \centering
          \resizebox{\linewidth}{!}{\includegraphics[width=0.41	\textwidth]{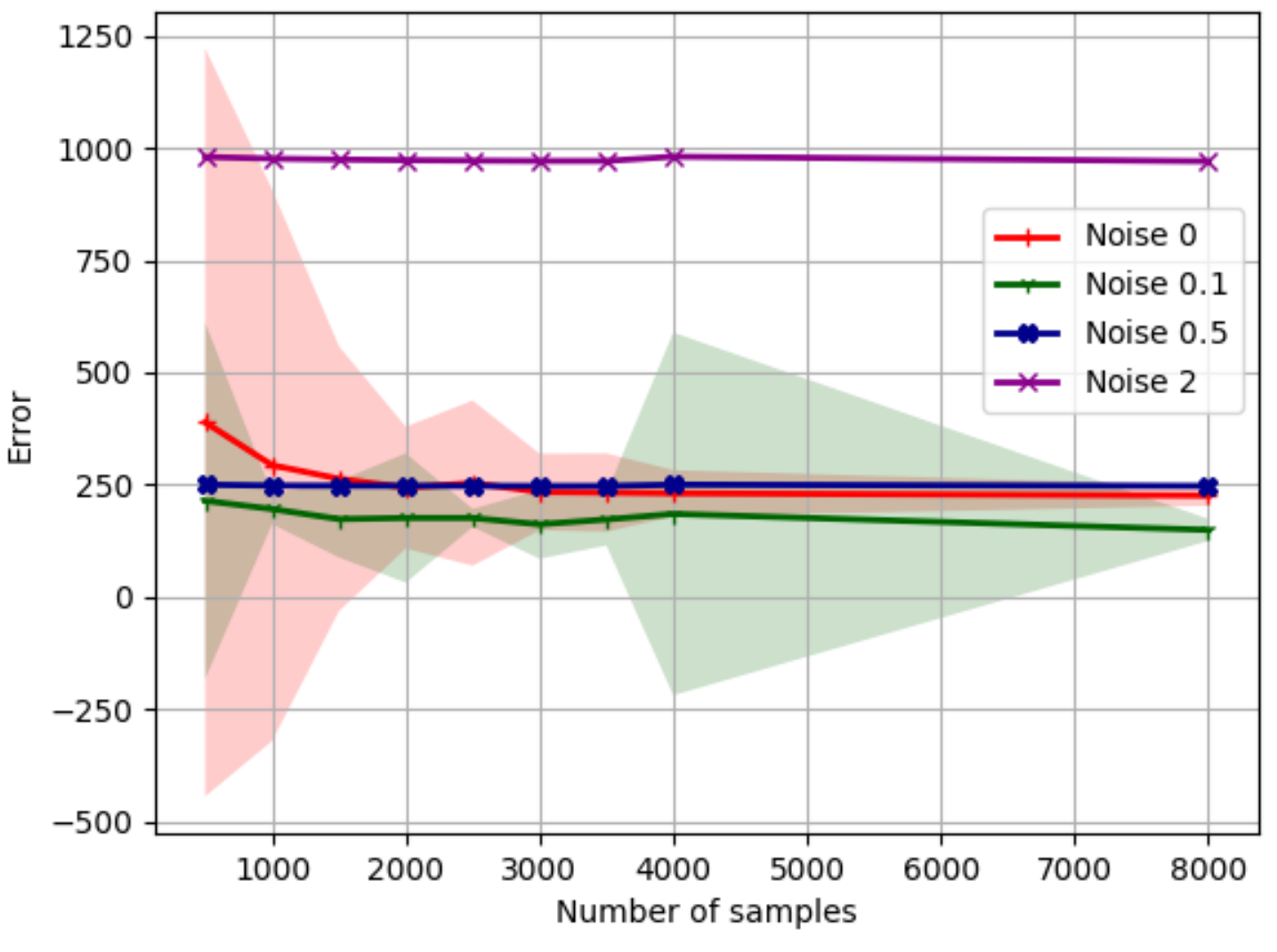}}  
          \caption{Reconstruction error}
          \label{fig:B}
     \end{subfigure}
    \caption{\small{Finding anchor set $\mathcal{A}$ of a $500\times 500$ matrix of rank $10$ on five noise levels and six different settings of the number of samples. In Figure (a), the label `Ours $x$' refers to the noise level $\mu=x$, e.g., `DCA 2' represents $\mu=2$. Similarly, the label `Noise x' in Figure (b) refers to noise level $\mu=x$. The shadow region refers to the variance of the error, e.g., the green region refers to the variance of reconstruction error with $\mu=0.5$.}}
      \label{fig:SNMF}
 \end{figure}

\section{Conclusion}\label{sec:conl}
In this paper, we have proposed a sublinear runtime classical algorithm to resolve the general minimum conical hull problem.  We first reformulated the general minimum conical hull problem as a sampling problem. We then exploited two sampling subroutines and proposed the general heuristic post-selection method to achieve low computational cost.   We theoretically analyzed the correctness and the computation cost of our algorithm. The proposed algorithm benefit numerous learning tasks that can be mapped to the general minimum conical hull problem, especially for tasks that need to process datasets on a huge scale.  There are two promising future directions. First, we will explore other advanced sampling techniques to further improve the polynomial dependence in the computation complexity.  Second, we will investigate whether there exist other fundamental learning models that  can be reduced to the general minimum conical hull problem. One of the most strongest candidates is the semi-definite programming solver.   

\section{Note Added}  Recently we became aware of an independent related work \cite{Yinan_Li_2019}, which employs the divide-and-conquer anchoring strategy to solve separable nonnegative matrix factorization problems. Since a major application of the general conical hull problem is to solve matrix factorization, their result can be treated as a special case of our study.  Moreover, our study adopts more advanced techniques and  provides a better upper complexity bounds than theirs.

\newpage
\appendix
We organize the supplementary material as follows. In Section \ref{sec:SM_BNS}, we detail the binary tree structure to support length-square sampling operations. We detail the inner product subroutine and the thin matrix-vector multiplication subroutine in Section \ref{Sec:samp_rout}. We then provide the proof of Theorem \ref{thm:post_sele} in Section \ref{sec:SM_gene_post}. Because the proof of Theorem \ref{thm:comp} cost employs the results of Theorem \ref{thm:main_corect}, we give the proof of Theorem \ref{thm:main_corect} in Section \ref{sec:SM_correct} and leave the proof of Theorem \ref{thm:comp} in Section \ref{sec:SM_compl}. 

\section{The Binary Tree   Structure for Length-square Sampling Operations}\label{sec:SM_BNS}
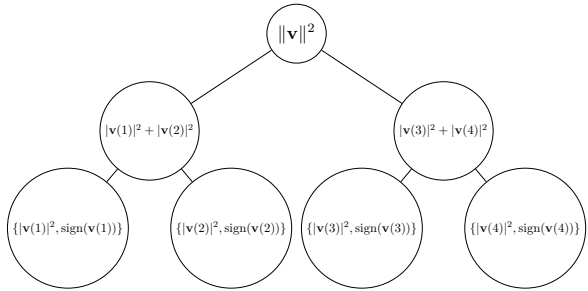
\begin{wrapfigure}{R}{0.5\textwidth}
\begin{adjustbox}{width=0.5\textwidth}
\begin{tikzpicture}
\node[circle,draw,scale=0.8](z){$\|\mathbf{v}\|^2$}
[sibling distance=45mm]
child{  node[circle,draw,scale=0.55](1){$|\mathbf{v}(1)|^2+|\mathbf{v}(2)|^2$}[sibling distance=25mm] child{node[circle,draw,scale=0.55] {$\{|\mathbf{v}(1)|^2, \text{sign}(\mathbf{v}(1))\}$}} child{node[circle,draw,scale=0.55] {$\{|\mathbf{v}(2)|^2, \text{sign}(\mathbf{v}(2))\}$}} }
  child{
node[circle,draw,scale=0.55](2){$|\mathbf{v}(3)|^2+|\mathbf{v}(4)|^2 $}
[sibling distance=25mm] child{node[circle,draw,scale=0.55] {$\{|\mathbf{v}(3)|^2, \text{sign}(\mathbf{v}(3))\}$}} child{node[circle,draw,scale=0.55] {$\{|\mathbf{v}(4)|^2, \text{sign}(\mathbf{v}(4))\}$}} };
\end{tikzpicture}   
\end{adjustbox}
\caption{\small{The BNS for $\mathbf{v}\in\mathbb{R}^4$. }}
\label{fig:BNS}
\end{wrapfigure}
As mentioned in the main text, a feasible solution to fulfill $\ell_2$  norm sampling operations  is the binary tree structure (BNS) to store data   \cite{kerenidis2017quantum}. Here we give the intuition about how BNS constructed for a vector. For ease of notations, we assume the given vector has size $4$ with $\mathbf{v}\in\mathbb{R}^{4}$. As demonstrated in Figure \ref{fig:BNS}, the root node records the square $\ell_2$ norm of $\mathbf{v}$. The $i$-th  leaf node records the $i$-th entry of $\mathbf{v}(i)$ and its square value, e.g., $\{|\mathbf{v}(i)|^2, \text{sign}(\mathbf{v}(i))\}$. Each internal node contains the sum of the values of its two immediate children. Such an architecture ensures the $\ell_2$ norm sampling opeartions.

\section{Two Sampling Subroutines}\label{Sec:samp_rout}
Here we introduce two sampling subroutines, the inner product subroutine and the thin matrix-vector multiplication subroutine \cite{tang2018quantum}, used in the proposed algorithm. 

\subsection{Inner product subroutine}\label{subsec:sampling_subrot}
 In our algorithm, the inner product subroutine is employed to obtain each entry of $\hat{\bm{q}}_t$ in parallel, i.e., $\hat{\bm{q}}_t(i)$ estimates $\tilde{\bm{q}}_t(i)\equiv\tilde{\bm{v}}^{(i)\top}\bm{H}_t$, where $\tilde{\bm{v}}^{(i)}\in\mathbb{R}^{n\times 1}$, $\bm{H}_t \equiv\bm{H}\bm{B}_t$, $\bm{H}\in\mathbb{R}^{n\times m}$ and $\bm{B}_t\in\mathbb{R}^{m\times 1}$.  
 Let $Z$ be a random variable that, for $j\in[n],l\in[m]$, takes value
 $$Z(j,l)=\frac{\|\tilde{\bm{v}}^{(i)}\|\|\bm{B}_t\|\bm{H}(j,l)}{\tilde{\bm{v}}^{(i)}(j)\bm{B}_t(l)}$$
 with probability     
 \begin{equation}\label{eqn:SM_aXb}
 \mathcal{P}_Z({Z}(j,l))=\frac{|\tilde{\bm{v}}^{(i)}(j)|^2|\bm{B}_t(l)|^2}{\|\tilde{\bm{v}}^{(i)}\|^2\|\bm{B}_t\|^2}~.
 \end{equation}

 
We can estimate $\tilde{\bm{q}}_t(i)$ using $Z$ as follows \cite{chia2018quantum}. Fix $\epsilon,\delta>0$. Let 
\begin{equation}
N_Z = N_p\times N_q\sim \mathcal{O}\left(\frac{\|\bm{H}\|_F\|\bm{B}_t\|\|\tilde{\bm{v}}^{(i)}\|}{\epsilon^2}\log(\frac{1}{\delta}) \right).
\end{equation}  
We first sample the distribution $\mathcal{P}_Z$ with $N_Z$ times to obtain a set of samples $\{z_i\}_{i=1}^{N_Z}$, followed by dividing them into $N_p$ groups, $\{S_1,\cdots,S_{N_p}\}$, where each group $S_i$ contains $N_q$ samples. Let $\bar{z}_i=\sum_{i=1}^{N_q}z_i/N_q$ be the empirical mean of the $i$-th group $S_i$, and let $\bar{z}^*$ be the median of $\{\bar{z}_i\}_{i=1}^{N_p}$. Then {\cite[Lemma 12]{gilyen2018quantum}} and {\cite[Lemma 7]{chia2018quantum}} guarantee that, with probability at least $1-\delta$, the following holds. 
 \begin{equation}\label{eq_mh05}
 | \bar{z}^* - \tilde{\bm{q}}_t(i) | \leq \epsilon.
 \end{equation}


The computational complexity of the inner product subroutine is: 
\begin{lem}[{\cite[Lemma 12]{gilyen2018quantum}} and {\cite[Lemma 7]{chia2018quantum}}]\label{lem:aXb}
Assume that the overall access to $\bm{H}$ is $L(\bm{H})$ and the query access to $\bm{B}_t$  and  $\tilde{\bm{V}}$ is $Q(\bm{B}_t)$ and $Q(\tilde{\bm{V}})$, respectively.
The runtime  complexity to yield Eqn.~(\ref{eq_mh05}) is $$\mathcal{O}\left(\frac{\|\bm{H}\|_F\|\bm{B}_t\|\|\tilde{\bm{v}}^{(i)}\|}{\epsilon^2}(L(\bm{H})+Q(\bm{B}_t)+Q(\tilde{\bm{V}}))\log(\frac{1}{\delta}) \right)~.$$
\end{lem}
\begin{proof}
We first recall the main result of Lemma 12 in \cite{gilyen2018quantum}. Given the overall access to $\bm{H}$ and query access to the matrix $\bm{\tilde{V}}$ and $\bm{B}_t$ with complexity $Q(\bm{\tilde{V}})$ and $Q(\bm{B}_t)$,  the inner product $\tilde{\bm{q}}_t(i)\equiv\tilde{\bm{v}}^{(i)\top}\bm{H}_t$ can be estimated to precision $\epsilon\|\bm{H}\|_F\|\tilde{\bm{v}}^{(i)}\|\|\bm{B}_t\|_F$ with probability at least $1-\delta$ in time $$\mathcal{O}\left(\frac{(L(\bm{H})+Q(\bm{B}_t) + Q(\bm{\tilde{V}}))}{\epsilon^2}\log(\frac{1}{\delta})\right)~.$$ With setting the precision to $\epsilon$ instead of $\epsilon\|\bm{H}\|_F\|\|\tilde{\bm{v}}^{(i)}\|\bm{B}_t\|_F$, it can be easily inferred that the runtime complexity to estimate $\tilde{\bm{q}}_t(i)\equiv\tilde{\bm{v}}^{(i)\top}\bm{H}_t$ with probability at least $1-\delta$ is 
$$\mathcal{O}\left(\frac{\|\bm{H}\|_F\|\bm{B}_t\|\|\tilde{\bm{v}}^{(i)}\|}{\epsilon^2}(L(\bm{H})+Q(\bm{B}_t)+Q(\tilde{\bm{V}}))\log(\frac{1}{\delta}) \right)~.$$
\end{proof}


\subsection{Thin matrix-vector multiplication subroutine}\label{sec:SM_samp}
Given a matrix $\tilde{\bm{V}} \in \mathbb{R}^{n\times k}$ and $\bm{\hat{q}}_t \in \mathbb{R}^{k}$} with the $\ell_2$ norm  sampling access,  the thin matrix-vector multiplication subroutine aims to output a sample  from $\mathcal{P}_{\tilde{\bm{V}}\bm{\hat{q}}_t}$. The implementation of the thin matrix-vector multiplication subroutine is  as follows \cite{gilyen2018quantum}.  

For each loop 
\begin{enumerate} 
	\item Sample a column index $j\in[k]$ uniformly. 
	\item Sample a row index $l\in[n]$ from distribution $\mathcal{P}_{\bm{\tilde{V}}^{\top}(:;j)}(l)=\frac{|\bm{\tilde{V}}^{\top}(l;j)|^2}{\|\bm{\tilde{V}}^{\top}(:;j)\|^2}$;
	\item Compute $|\bm{\tilde{V}}(l;:)\bm{\hat{q}}_t|^2$ 
	\item  Output $l$ with probability $\frac{|\bm{\tilde{V}}(l;:)\bm{\hat{q}}_t|^2}{\|\bm{\tilde{V}}(l;:)\|^2\|\bm{\hat{q}}_t\|^2}$ or restart to sample $(j,l)$ again (Back to step 1).
\end{enumerate}
We execute the above loop until a sample $l$ is successfully  output.

The complexity of the thin matrix-vector multiplication subroutine, namely, the required number of loops,    is: 
\begin{lem}[{\cite[Lemma 12]{gilyen2018quantum}} and {\cite[Proposition 6.4]{tang2018quantum}}]\label{lem:Ab}
Let $\tilde{\bm{V}} \in \mathbb{R}^{n\times k}$ and $\bm{\hat{q}}_t \in \mathbb{R}^{k}$. Given $\ell_2$ norm sampling access to $\tilde{\bm{V}}$, we can length-square sample from  $\mathcal{P}_{\tilde{\bm{V}}\bm{\hat{q}}_t}$ in the expected runtime complexity 
 $$\mathcal{O}\left(\frac{k\|\bm{\hat{q}}_t\|^2}{\|\tilde{\bm{V}}\bm{\hat{q}}_t\|^2}(S(\tilde{\bm{V}})+kQ(\tilde{\bm{V}})\right)~.$$ 
	\end{lem}

\section{General Heuristic Post-selection (Proof of Theorem \ref{thm:post_sele})}\label{sec:SM_gene_post}
Recall that the anchor $\mathcal{A}_t$ is defined as $$\mathcal{A}_t =  \arg\min_{i\in[n_{\scalebox{.49}{Y}}]}\left(\mathcal{P}_{\bm{\hat{Y}}_t}(i)-\hat{\xi}_t\max_{j\in[n_{\scalebox{.49}{X}}]}(\mathcal{P}_{\bm{\hat{X}}_t}(j))\right)_+~.$$ 
The goal of the general heuristic post-selection is approximating ${\mathcal{A}}_t$ by sampling distributions $\mathcal{P}_{\bm{\hat{X}}_t}$ and $\mathcal{P}_{\bm{\hat{Y}}_t}$ with $\mathcal{O}(polylog(\max\{n_{\scalebox{.49}{X}}, n_{\scalebox{.49}{Y}}\}))$ times, since the acquisition of the explicit form of $\mathcal{P}_{\bm{\hat{X}}_t}$ and $\mathcal{P}_{\bm{\hat{Y}}_t}$  requires $\mathcal{O}(poly(n_{\scalebox{.49}{X}},n_{\scalebox{.49}{Y}}))$ computation complexity and collapses the desired speedup. Let  $\{\bm{x}_i\}_{i=1}^{N_{\scalebox{.49}{X}}}$ be $N_{\scalebox{.49}{X}}$  examples independently sampled from $\mathcal{P}_{\bm{\hat{X}}_t}$ with $\mathcal{P}_{\bm{\hat{X}}_t}(\bm{x}=z)=|\bm{{\hat{X}}}_t(z)|^2/\|\bm{{\hat{X}}}_t\|^2$ and $N_{{\scalebox{.49}{X}},z}$ be the number of examples taking value of $z\in[n_{\scalebox{.49}{X}}]$ with $\sum_{z=1}^{n_{\scalebox{.49}{X}}}N_{{\scalebox{.49}{X}},z}=N_{{\scalebox{.49}{X}}}$.  Similarly, let  $\{\bm{y}_i\}_{i=1}^{N_{\scalebox{.49}{Y}}}$ be $N_{\scalebox{.49}{Y}}$ examples independently sampled from $\mathcal{P}_{\bm{\hat{Y}}_t}$ with  $\mathcal{P}_{\bm{\hat{Y}}_t}(\bm{y}=z)=|\bm{{\hat{Y}}}_t(z)|^2/\|\bm{{\hat{Y}}}_t\|^2$ and $N_{{\scalebox{.49}{Y}},z}$ be the number of examples taking value of $z\in[n_{\scalebox{.49}{Y}}]$ with $\sum_{z=1}^{n_{\scalebox{.49}{Y}}}N_{{\scalebox{.49}{Y}},z}=N_{{\scalebox{.49}{Y}}}$.  Denote that $w_{\scalebox{.49}{X}}$ is the total number of different  indexes after sampling $\mathcal{P}_{\bm{\hat{X}}_t}$ with  $N_{\scalebox{.49}{X}}$ times,   $I_{\scalebox{.49}{X},v}$ and $J_{\scalebox{.49}{X},v}$ are two indexes corresponding to  the $v$-th largest value among $\{N_{\scalebox{.49}{X},z}\}_{z=1}^{w_{\scalebox{.49}{X}}}$ and $\mathcal{P}_{\bm{\hat{X}}_t}$ with $w_{\scalebox{.49}{X}}\leq n_{\scalebox{.49}{X}}$ and $v\in[w_{\scalebox{.49}{X}}]$, respectively\footnote{Due to the limited sampling times, the sampled results $\{N_{\scalebox{.49}{X},1},N_{\scalebox{.49}{X},2},...,N_{\scalebox{.49}{X},w_{\scalebox{.49}{X}}}\}$ (or $\{N_{\scalebox{.49}{Y},1},N_{\scalebox{.49}{Y},2},...,N_{\scalebox{.49}{Y},w_{\scalebox{.49}{X}}}\}$) may occupy a small portion of the all $n_{\scalebox{.49}{X}}$ (or $n_{\scalebox{.49}{Y}}$) possible results.}. Similarly, denote that $w_{\scalebox{.49}{Y}}$ is the total number of distinguished indexes after sampling $\mathcal{P}_{\bm{\hat{Y}}_t}$  with  $N_{\scalebox{.49}{Y}}$ times, the indexes $I_{\scalebox{.49}{Y},v}$  and $J_{\scalebox{.49}{Y},v}$  are  $v$-th largest value among $\{N_{\scalebox{.49}{Y},z}\}_{z=1}^{w_{\scalebox{.49}{Y}}}$ and $\mathcal{P}_{\bm{\hat{Y}}_t}$ with $w_{\scalebox{.49}{Y}}\leq n_{\scalebox{.49}{Y}}$ and $v\in[w_{\scalebox{.49}{Y}}]$, respectively.   In particular, we have $$I_{\scalebox{.49}{X},1}=\arg\max_z N_{\scalebox{.49}{X},z}, ~I_{\scalebox{.49}{Y},1}=\arg\max_z N_{\scalebox{.49}{Y},z},$$ $$J_{\scalebox{.49}{X},1}= \arg\max_z \mathcal{P}_{\bm{\hat{X}}_t}(\bm{x}=z),~ \text{and}~ J_{\scalebox{.49}{Y},1}= \arg\max_z \mathcal{P}_{\bm{\hat{Y}}_t}(\bm{y}=z) ~.$$

The procedure of the general heuristic post-selection is as follows.
\begin{enumerate}
	\item Sample $\mathcal{P}_{\bm{\hat{X}}_t}$ with $N_{\scalebox{.49}{X}}$ times and order the sampled items to obtain $\mathcal{X}=\{I_{\scalebox{.49}{X},1}, I_{\scalebox{.49}{X},2},...,I_{\scalebox{.49}{X},w_{\scalebox{.49}{X}}}\}$ with $w_{\scalebox{.49}{X}}\leq n_{\scalebox{.49}{X}}$;
	\item Query the distribution $\mathcal{P}_{\bm{\hat{X}}_t}$ to obtain the value  $C_{\bm{\hat{X}}_t}^*$ with  $C_{\bm{\hat{X}}_t}^*=\mathcal{P}_{\bm{\hat{X}}_t}(\bm{x}=I_{\scalebox{.49}{X},1})$;
	\item Sample $\mathcal{P}_{\bm{\hat{Y}}_t}$ with $N_{\scalebox{.49}{Y}}$ times and order the sampled items to obtain $\mathcal{Y}=\{I_{\scalebox{.49}{Y},1}, I_{\scalebox{.49}{Y},2},...,I_{\scalebox{.49}{Y},w_{\scalebox{.49}{Y}}}\}$ with $w_{\scalebox{.49}{Y}}\leq n_{\scalebox{.49}{Y}}$.
	\item Locate $\hat{\mathcal{A}}_t$ with $\hat{\mathcal{A}}_t= I_{\scalebox{.49}{Y},v^*}$ and $I_{\scalebox{.49}{Y},v^*}= \arg \min_{z\in\mathcal{Y}}\left(\frac{N_{\scalebox{.49}{Y},z}}{N_{\scalebox{.49}{Y}}}-\hat{\xi}_tC_{\bm{\hat{X}}_t}^* \right)_+$. 
\end{enumerate}

An immediate observation is that, with $N_{\scalebox{.49}{X}},~N_{\scalebox{.49}{Y}}\rightarrow \infty$, we have $\hat{\mathcal{A}}_t=\mathcal{A}_t$, where $$\mathcal{P}_{\bm{\hat{X}}_t}(\bm{x}=I_{\scalebox{.49}{X},1})=\max_{J_{X,v}\in [w_{\scalebox{.49}{X}}]}(\mathcal{P}_{\bm{\hat{X}}_t}(J_{\scalebox{.49}{X},v})),~ \frac{N_{\scalebox{.49}{Y},I_{\scalebox{.49}{Y},v^*}}}{N_{\scalebox{.49}{Y}}}=\mathcal{P}_{\bm{\hat{Y}}_t}(\bm{y}=\mathcal{A}_t),~w_{\scalebox{.49}{X}}=n_{\scalebox{.49}{X}},~w_{\scalebox{.49}{Y}}=n_{\scalebox{.49}{Y}}~.$$

The general heuristic post-selection  is guaranteed by the following theorem (the formal description of  Theorem 5). 
\begin{thm}[(Formal) General heuristic post-selection]\label{thm:post_sele1}
	Assume that $\mathcal{P}_{\bm{\hat{X}}_t}$ and $\mathcal{P}_{\bm{\hat{Y}}_t}$ are multinomial distributions.  	 If $\mathcal{P}_{\bm{\hat{X}}_t}(J_{\scalebox{.49}{X},1}) \geq \varepsilon_T$, $\left|\mathcal{P}_{\bm{\hat{X}}_t}(J_{\scalebox{.49}{X},1}) - \mathcal{P}_{\bm{\hat{X}}_t}(J_{\scalebox{.49}{X},2})\right| > \varepsilon$ and $\left|\mathcal{P}_{\bm{\hat{Y}}_t}(J_{\scalebox{.49}{Y},v^*\pm 1}) - \mathcal{P}_{\bm{\hat{Y}}_t}(J_{\scalebox{.49}{Y},v^*})\right|> \varepsilon$ for the constants $\varepsilon_T$ and $\varepsilon$ with $\mathcal{A}_t=J_{\scalebox{.49}{Y},v^*}$,  then for any $\delta > 0$, $N_{\scalebox{.49}{X}}\sim \mathcal{O}(\max\{1/\varepsilon, 1/\varepsilon_T\}))$ and $N_{\scalebox{.49}{Y}}\sim  \mathcal{O}(\max\{1/\varepsilon, 1/(\hat{\xi}_t\varepsilon_T)\}))$, we have $\hat{\mathcal{A}}_t=\mathcal{A}_t$ with a probability at least $1 -\delta$. 
\end{thm} 

\textit{Remark.} The physical meaning of $\varepsilon$ can be treated as the threshold of the `near-anchor', that is, when the distance of a data point and the anchor point after projection is  within the threshold $\varepsilon$, we say the data point can be treated as anchors. The real anchor set $\mathcal{A}$ therefore should be expanded and include these near anchors. In other words, $\varepsilon$ and $\varepsilon_T$ can be manually controllable. The parameter $\hat{\xi}_t$ is bounded as follows. 
$$\hat{\xi}_t=\frac{\|\bm{\hat{Y}}_t\|}{\|\bm{\hat{X}}_t\|}\leq \kappa_{\scalebox{.49}{X}}\|\hat{\bm{Y}}\|_F~.$$
In this work, we set $1/\varepsilon\sim \mathcal{O}(polylog(n_{\scalebox{.49}{X}},n_{\scalebox{.49}{Y}})$ and $1/\varepsilon_T\sim \mathcal{O}(polylog(n_{\scalebox{.49}{X}},n_{\scalebox{.49}{Y}})$, which gives $N_{\scalebox{.49}{X}},~N_{\scalebox{.49}{Y}}\rightarrow \mathcal{O}(\kappa_{\scalebox{.49}{X}},  polylog(\max\{n_{\scalebox{.49}{X}},n_{\scalebox{.49}{Y}}\}))$. 

 An equivalent statement of Theorem \ref{thm:post_sele1} is: 
\begin{prob}\label{prob:post_selec}
How many samples, $N_{\scalebox{.49}{X}}$ and $N_{\scalebox{.49}{Y}}$, are required to  guarantee $I_{\scalebox{.49}{X},1}=J_{\scalebox{.49}{X},1}$  and $I_{\scalebox{.49}{Y},v^*}=J_{\scalebox{.49}{Y},v^*}$, where $J_{\scalebox{.49}{X},1}=\arg\max_{J_{{\scalebox{.49}{X}},v}\in[n_{\scalebox{.49}{X}}]}(\mathcal{P}_{\bm{\hat{X}}_t}(J_{{\scalebox{.49}{X}},v^*}))$  and  $\mathcal{A}_t=J_{\scalebox{.49}{Y},v^*}$.  	
\end{prob}

We use  Breteganolle-Huber-Carol inequality   \cite{van1996weak} to  prove Theorem \ref{thm:post_sele1}, i.e.,
 \begin{lem}[Breteganolle-Huber-Carol inequality  \cite{van1996weak}]
 	Let $\mathcal{P}_D$ be a multinomial distribution with $l$ event probabilities $\{\mathcal{P}_D(i)\}_{i=1}^l$. We randomly sample $N$ events from $\mathcal{P}_D$   and let $N_i$ be the number of event $i$ appeared. Then, the following holds with a probability at least $1-\delta$ for any $\delta>0$,
 	\begin{equation}\label{eqn:BHC}
 	\mathcal{P}\left(\sum_{i=1}^{l}\left|\frac{N_i}{N}-p_i\right|\geq\lambda\right) \leq 2^{l}\exp{\left(\frac{-N\lambda^2}{2}\right)}~.
 	\end{equation}	
 \end{lem}
 \begin{proof}[Proof of Theorem \ref{thm:post_sele1}]
 The proof is composed of two parts. The first part is to prove that the index   $I_{\scalebox{.49}{X},1}$ with $I_{\scalebox{.49}{X},1}=J_{\scalebox{.49}{X},1}$ can be determined with sampling complexity $\mathcal{O}(1/(\varepsilon^2\varepsilon_T^2))$. The second part is to prove that the index $I_{\scalebox{.49}{Y},w_{\scalebox{.49}{Y}}}$ with $I_{\scalebox{.49}{Y},w_{\scalebox{.49}{Y}}}=\mathcal{A}_t$ can be determined with sampling complexity $\mathcal{O}(1/(\varepsilon^2\varepsilon_T^2))$. 
 
For the first part, we split the set $\mathcal{X}$ into two subsets $\mathcal{X}_1$ and $\mathcal{X}_2$, i.e., $\mathcal{X}_1=\{I_{\scalebox{.49}{X},1}\}$ and  $\mathcal{X}_2=\{I_{\scalebox{.49}{X},2},...,I_{\scalebox{.49}{X},w_{\scalebox{.49}{X}}}\}$. The decomposition of $\mathcal{X}$ into two subsets is equivalent to setting $l=2$ in Eqn.~(\ref{eqn:BHC}). In particular, we have $N_1=N_{\scalebox{.49}{X},I_{\scalebox{.49}{X},1}}$ and $N_2=\sum_{v=2}^{w_{\scalebox{.49}{X}}}N_{\scalebox{.49}{X},I_{\scalebox{.49}{X},v}}$. The Breteganolle-Huber-Carol inequality yields 
 \begin{equation}
 	\mathcal{P}\left(\left|\frac{N_{1}}{N_{\scalebox{.49}{X}}}-\mathcal{P}_{\bm{\hat{X}}_t}(\mathcal{X}_1)\right|+\left|\left(\frac{N_{2}}{N_{\scalebox{.49}{X}}}-\mathcal{P}_{\bm{\hat{X}}_t}(\mathcal{X}_2)\right)\right|\geq\lambda\right) \leq 4\exp{\left(\frac{-N_{\scalebox{.49}{X}}\lambda^2}{2}\right)}~.
 \end{equation}
The above inequality implies that, when $\delta=4\exp{\left(\frac{-N_{\scalebox{.49}{X}}\lambda^2}{2}\right)}$, we have
 \begin{eqnarray}\label{eqn:SM_postse_1}
\left|\frac{N_{\scalebox{.49}{X},I_{\scalebox{.49}{X},1}}}{N_{\scalebox{.49}{X}}}-\mathcal{P}_{\bm{\hat{X}}_t}(I_{\scalebox{.49}{X},1})\right| &\leq&\left|\frac{N_{1}}{N_{\scalebox{.49}{X}}}-\mathcal{P}_{\bm{\hat{X}}_t}(\mathcal{X}_1)\right|+\left|\left(\frac{N_{2}}{N_{\scalebox{.49}{X}}}-\mathcal{P}_{\bm{\hat{X}}_t}(\mathcal{X}_2)\right)\right| \\
&\leq& \sqrt{\frac{2\log{(4/\delta)}}{N_{\scalebox{.49}{X}}}}~,
 \end{eqnarray}
 with probability at least $1-\delta$. 

The assumption $\mathcal{P}_{\bm{\hat{X}}_t}(J_{\scalebox{.49}{X},1}) \geq \varepsilon_T$ guarantees that $J_{\scalebox{.49}{X},1}\in \mathcal{X}$ by sampling $\mathcal{P}_{\bm{\hat{X}}_t}$ with $\mathcal{O}(1/\varepsilon_T)$ times.   In addition, since we have assumed that  $\left|\mathcal{P}_{\bm{\hat{X}}_t}(J_{\scalebox{.49}{X},1}) - \mathcal{P}_{\bm{\hat{X}}_t}(J_{\scalebox{.49}{X},2})\right| > \varepsilon$, it can be easily inferred that, when $N_{\scalebox{.49}{X}}\geq \sqrt{\frac{2\log{(4/\delta)}}{\varepsilon^2}}$ with a probability at least $1-\delta$, there is one and only one value $\frac{N_{\scalebox{.49}{X},I_{\scalebox{.49}{X},1}}}{N_{\scalebox{.49}{X}}}$ that is in the $\varepsilon$-neighborhood of $\mathcal{P}_{\bm{\hat{X}}_t}(I_{\scalebox{.49}{X},1})$. We therefore conclude that  $I_{\scalebox{.49}{X},1}=J_{\scalebox{.49}{X},1}$ can be guaranteed by sampling $\mathcal{P}_{\bm{\hat{X}}_t}$ with $N_{\scalebox{.49}{X}}\geq\max\{\sqrt{\frac{2\log{(4/\delta)}}{\varepsilon^2}}, \frac{1}{\varepsilon_T} \}$.
 
 For the second part, we split the set $\mathcal{Y}$ into three subsets $\mathcal{Y}_1$,  $\mathcal{Y}_2$ and $\mathcal{Y}_3$, i.e., $\mathcal{Y}_1=\{I_{\scalebox{.49}{Y},1},...,I_{\scalebox{.49}{Y},v^*-1}\}$,  $\mathcal{Y}_2=\{I_{\scalebox{.49}{Y},v^*}\}$, and  $\mathcal{Y}_3=\{I_{\scalebox{.49}{Y},v^*+1},...,I_{\scalebox{.49}{Y},w_{\scalebox{.49}{Y}}}\}$.   Analogous to the above case, 
the decomposition of $\mathcal{Y}$ into three subsets for the case of sampling $\mathcal{P}_{\bm{\hat{Y}}_t}$ is equivalent to setting $l=3$ in Eqn.~(\ref{eqn:BHC}). In particular, we have $N_1=\sum_{v=1}^{v^*} N_{\scalebox{.49}{Y},I_{\scalebox{.49}{Y},v}}$,  $N_2=N_{\scalebox{.49}{Y},I_{\scalebox{.49}{Y},v^*}}$, and $N_3=\sum_{v=v^*+1}^{w_{\scalebox{.49}{Y}}}N_{\scalebox{.49}{Y},I_{\scalebox{.49}{Y},v}}$. The Breteganolle-Huber-Carol inequality yields 
  \begin{equation}
  \mathcal{P}\left(\left|\frac{N_{1}}{N_{\scalebox{.49}{Y}}}-\mathcal{P}_{\bm{\hat{Y}}_t}(\mathcal{Y}_1)\right|+\left|\left(\frac{N_2}{N_{\scalebox{.49}{Y}}}-\mathcal{P}_{\bm{\hat{Y}}_t}(\mathcal{Y}_2)\right)\right| + \left| \left(\frac{N_3}{N_{\scalebox{.49}{Y}}}-\mathcal{P}_{\bm{\hat{Y}}_t}(\mathcal{Y}_3)\right) \right| \geq\lambda\right)
  \leq  8\exp{\left(\frac{-N_{\scalebox{.49}{Y}}\lambda^2}{2}\right)}.
 \end{equation}
 
 The above inequality implies that, when $\delta=4\exp{\left(\frac{-N_{\scalebox{.49}{Y}}\lambda^2}{2}\right)}$, we have  
 \begin{eqnarray}\label{eqn:SM_postse_2}
\left|\frac{N_{\scalebox{.49}{Y},I_{\scalebox{.49}{Y},v^*}}}{N_{\scalebox{.49}{Y}}}-\mathcal{P}_{\bm{\hat{Y}}_t}(I_{\scalebox{.49}{Y},v^*})\right|
&\leq & \left|\frac{N_{1}}{N_{\scalebox{.49}{Y}}}-\mathcal{P}_{\bm{\hat{Y}}_t}(\mathcal{Y}_1)\right|+\left|\left(\frac{N_2}{N_{\scalebox{.49}{Y}}}-\mathcal{P}_{\bm{\hat{Y}}_t}(\mathcal{Y}_2)\right)\right| + \left| \left(\frac{N_3}{N_{\scalebox{.49}{Y}}}-\mathcal{P}_{\bm{\hat{Y}}_t}(\mathcal{Y}_3)\right) \right|   \nonumber\\
\leq && \sqrt{\frac{2\log{(8/\delta)}}{N_{\scalebox{.49}{Y}}}}~,
 \end{eqnarray}
 with probability at least $1-\delta$. 
  
  Since we have assumed $\mathcal{P}_{\bm{\hat{X}}_t}(J_{\scalebox{.49}{X},1}) \geq \varepsilon_T$, we have $\mathcal{P}_{\bm{\hat{Y}}_t}(J_{\scalebox{.49}{Y},v^*}) \geq \hat{\xi}_t \varepsilon_T$.  In other words, $\mathcal{P}_{\bm{\hat{Y}}_t}(J_{\scalebox{.49}{Y},v^*}) \geq \hat{\xi}_t \varepsilon_T$  guarantees that $I_{\scalebox{.49}{Y},v^*}=J_{\scalebox{.49}{Y},v^*}$  by sampling $\mathcal{P}_{\bm{\hat{Y}}_t}$ with $\mathcal{O}(1/(\hat{\xi}_t \varepsilon_T))$ times. In addition,   the assumption $\left|\mathcal{P}_{\bm{\hat{Y}}_t}(J_{\scalebox{.49}{Y},v^*\pm 1}) - \mathcal{P}_{\bm{\hat{Y}}_t}(J_{\scalebox{.49}{Y},v^*})\right|> \varepsilon$ leads to that, when $N_{\scalebox{.49}{Y}}\geq \max\{\sqrt{\frac{2\log{(8/\delta)}}{\varepsilon^2}}, \frac{1}{\hat{\xi}_t\varepsilon_T}\}$ with a probability at least $1-\delta$, there is one and only one value $\frac{N_{\scalebox{.49}{Y},I_{\scalebox{.49}{Y},v^*}}}{N_{\scalebox{.49}{Y}}}$ that is in the $\varepsilon$-neighborhood of $\mathcal{P}_{\bm{\hat{Y}}_t}(I_{\scalebox{.49}{Y},v^*})$.  We therefore  conclude that  $I_{\scalebox{.49}{Y},v^*}=J_{\scalebox{.49}{Y},v^*}$ with $J_{\scalebox{.49}{Y},v^*}=\hat{\mathcal{A}}_t$. 
  
  Combing the results of two parts together, it can be easily inferred that, with the sampling complexity $\mathcal{\tilde{O}}(\max\{\frac{1}{\varepsilon}, \frac{1}{\hat{\xi}_t \varepsilon_T}\})$, we have $\mathcal{A}_t=\hat{\mathcal{A}}_t $ with probability $1-\delta$. 
\end{proof}

\section{Proof of Theorem \ref{thm:main_corect}}\label{sec:SM_correct}
In this section, we give  the proof of Theorem  \ref{thm:main_corect}. We left the detailed proof of Theorem \ref{thm:main_cor1} and \ref{thm:total_va_aXb} in subsection \ref{subsecion:proof_corr_1} and \ref{subsec:C.2}, respectively. 
\begin{proof}[ Proof of Theorem  \ref{thm:main_corect}]

  Recall that $\bm{\hat{H}}_t=\tilde{\bm{V}}\bm{\hat{q}}_{\scalebox{.49}{H,t}}$, and $\bm{\hat{q}}_{\scalebox{.49}{H,t}}$ is an approximation of $\bm{\tilde{q}}_{\scalebox{.49}{H,t}}=\tilde{\bm{V}}^{\top}\bm{H}_t$. The triangle inequality yields
 \begin{equation}\label{eqn:correctness_QI}
\|\mathcal{P}_{\bm{\hat{H}}_t}-\mathcal{P}_{\bm{H}_t}\|_{TV}	\leq  \|\mathcal{P}_{\bm{\hat{H}}_t}-\mathcal{P}_{\bm{\tilde{H}}_t}\|_{TV}	
+\|\mathcal{P}_{\bm{\tilde{H}}_t}-\mathcal{P}_{\bm{H}_t}\|_{TV}~,
 \end{equation}
 where $ \|\mathcal{Q}\|_{TV}$ is the total variation distance of $\mathcal{Q}$. In the following, we bound the two terms on the right-hand side of Eqn.~(\ref{eqn:correctness_QI}) respectively.

  \textit{\underline{Correctness of  $\|\mathcal{P}_{\bm{\tilde{H}}_t}-\mathcal{P}_{\bm{{H}}_t}\|_{TV} $}.} The goal here is to prove that  
 \begin{equation}\label{eq_mh01}
 \|\mathcal{P}_{\bm{\tilde{H}}_t}-\mathcal{P}_{\bm{{H}}_t}\|_{TV} \leq \frac{\epsilon}{2}.
 \end{equation} 
  
By Lemma~\ref{lem:tv_QI} below, Eqn~(\ref{eq_mh01}) follows if the following inequality holds:
  \begin{equation}\label{eqn:thm_13_core}
  	\left\|\bm{\tilde{H}}_t-\bm{H}_t\right\| \leq \frac{\epsilon}{4}\|\bm{H}_t\|.
  \end{equation}
Finally, the inequality in Eqn.~(\ref{eqn:thm_13_core}) is guaranteed to hold because of Theorem~\ref{thm:main_cor1}.
%
\begin{lem}[Lemma 6.1, \cite{tang2018quantum}] \label{lem:tv_QI}
For $\bm{x}, \bm{y} \in \mathbb{R}^n$ satisfying $\|\bm{x}-\bm{y}\|\leq \epsilon$, the corresponding distributions $\mathcal{P}_{\bm{x}}$ and $\mathcal{P}_{\bm{y}}$ satisfy $\|\mathcal{P}_{\bm{x}}, \mathcal{P}_{\bm{y}}\|_{TV}\leq \frac{2\epsilon}{\|\bm{x}\|}$.
\end{lem}
 

\begin{thm}\label{thm:main_cor1}
Let the rank and the condition number of $\bm{H}\in\mathbb{R}^{n\times m}$ be $k$ and $\kappa$, respectively.  Fix $$s=\frac{85^2k^3\kappa^4\ln(8n/\eta)\|\bm{H}\|_F^2}{9\epsilon^2 }~.$$ Then, Algorithm \ref{alg:meta} yields $\|\bm{\tilde{H}}_t-\bm{H}_t\| \leq \frac{\epsilon\|\bm{H}_t\|}{4}$ with probability at least $(1-\eta)$.
\end{thm}

 \textit{\underline{Correctness of $\|\mathcal{P}_{\bm{\hat{H}}_t}-\mathcal{P}_{\bm{\tilde{H}}_t}\|_{TV} $}.} Analogous to the above part, we bound   
 \begin{equation}\label{eq_mh02}
 \|\bm{\hat{H}}_t-\bm{\tilde{H}}_t\|\leq \frac{\epsilon}{4}\|\bm{\tilde{H}}_t\|,
 \end{equation}
 to yield
 \begin{equation}\label{eq_mh03}
  \|\mathcal{P}_{\bm{\hat{H}}_t}-\mathcal{P}_{\bm{\tilde{H}}_t}\|_{TV}	\leq \frac{\epsilon}{2}.
 \end{equation}
And Eqn.~(\ref{eq_mh02}) can be obtained by the following theorem.
 \begin{thm}\label{thm:total_va_aXb} 
Let the rank of $\bm{H}\in\mathbb{R}^{n\times m}$ be $k$. Set the number of samplings in the inner product subroutine as $$	N_Z\sim \mathcal{O}\left(\frac{(4+\epsilon)\sqrt{k}\|\bm{H}\|_F\|\bm{B}_t\|\|\tilde{\bm{v}}^{(i)}\|}{4\epsilon}\log(1/\delta)\right)~.$$   Then, Algorithm \ref{alg:meta} yields  $\|\bm{\hat{H}}_t-\bm{\tilde{H}}_t\|\leq \frac{\epsilon}{4}\|\bm{\tilde{H}}_t\|$ with at least $1-\delta$ success probability.
 \end{thm}
 Finally, Eqn.~(\ref{eqn:correctness_QI}) holds combining Eqn.~(\ref{eq_mh01}) with Eqn.~(\ref{eq_mh03}).
  \end{proof}
 

\subsection{ Proof of Theorem \ref{thm:main_cor1}}\label{subsecion:proof_corr_1}
Due to $\bm{H}_t=\bm{V}\bm{V}^{\top}\bm{H}\bm{B}_t$, we have 
\begin{eqnarray}\label{eqn:correct}
  \|\bm{\tilde{H}}_t-\bm{H}_t\| &=& \|\tilde{\bm{V}} \tilde{\bm{V}}^{\top}\bm{H}\bm{B}_t -\bm{V}\bm{V}^{\top}\bm{H}\bm{B}_t\| \nonumber\\
  &\leq& \left\|\left(\sum_{i=1}^k \tilde{\bm{v}}^{(i)}\tilde{\bm{v}}^{(i)\top}  - \Pi_{(\bm{H})}\right)\right\|_2\|\bm{H}\bm{B}_t\|\nonumber\\
  &\leq & \left\|\left(\sum_{i=1}^k \tilde{\bm{v}}^{(i)}\tilde{\bm{v}}^{(i)\top}  - \Pi_{(\bm{H})}\right)\right\|_2\|\bm{H}\|_2~,
\end{eqnarray}
where  $\Pi_{(\bm{H})} = \sum_i \bm{v}^{(i)} \bm{v}^{(i)\top}$ and  $\bm{v}^{(i)}$ is the left singular vectors of $\bm{H}$.
The first inequality of Eqn.~(\ref{eqn:correct}) is obtained by exploiting the submultiplicative property of spectral norm \cite{horn2012matrix}, i.e., for any matrix $\bm{M}\in\mathbb{R}^{n\times m}$ and any vector $\bm{z}\in\mathbb{R}^m$, we have  $\|\bm{M}\bm{z}\| \leq \|\bm{M}\|_2\|\bm{z}\|$. The second inequality of Eqn.~(\ref{eqn:correct}) comes from the submultiplicative property of spectral norm and  $\|\bm{B}_t\|=1$.  
To achieve $\|\bm{\tilde{H}}_t-\bm{H}_t\| \leq \frac{\epsilon}{4}\|\bm{H}_t\|$ in Eqn.~(\ref{eqn:thm_13_core}), Eqn.~(\ref{eqn:correct}) indicates that the approximated left singular matrix $\tilde{\bm{V}}$ should  satisfy \begin{eqnarray}\label{eqn:c.1_correct}
	\left\|\sum_{i=1}^k  \tilde{\bm{v}}^{(i)}\tilde{\bm{v}}^{(i)\top}  - \Pi_{(\bm{H})}\right\|_2\leq \frac{\epsilon}{4}~.
\end{eqnarray}

The spectral norm  $\|\sum_{i=1}^k \tilde{\bm{v}}^{(i)}\tilde{\bm{v}}^{(i)\top} - \Pi_{(\bm{H})}\|_2$ can be quantified as:
\begin{thm}\label{thm:main}
	Suppose that the rank of  $\bm{H}$  is $k$, and $\tilde{\bm{v}}^{(i)}$ refers to a approximated singular vector of $\bm{H}$ such that 
	\begin{equation}\label{eqn:thm1-1}
		\left|\tilde{\bm{v}}^{(i)\top}\tilde{\bm{v}}^{(j)} - \delta_{ij} \right| \leq \alpha \leq \frac{1}{4k}~.
	\end{equation}
	Then, we have
	\begin{equation}\label{eqn:thm1}
		\left\| \sum_{i=1}^k \tilde{\bm{v}}^{(i)}\tilde{\bm{v}}^{(i)\top} - \Pi_{(\bm{H})} \right\|_2 \leq   \frac{ 17k\alpha}{3}~.
	\end{equation}
\end{thm}
	The proof of Theorem \ref{thm:main} is given in Subsection \ref{subsec:D1}.
	
Theorem \ref{thm:main} implies that to achieve Eqn.~(\ref{eqn:c.1_correct}) (or equivalently, Eqn.~(\ref{eqn:thm_13_core})), we should bound  $\alpha$ as 
\begin{equation}\label{eqn:alpha_error}
	\left\| \sum_{i=1}^k \tilde{\bm{v}}^{(i)}\tilde{\bm{v}}^{(i)\top} - \Pi_{(\bm{H})} \right\|_2 \leq \frac{17k\alpha}{3} \leq \frac{\epsilon}{4}~.	
\end{equation}
We use the following lemma to give an explicit representation of $\alpha$ by the sampled matrix $\bm{R}$ and $\bm{C}$, 
\begin{lem}\label{lem:3.3}
	Suppose that $\bm{{\omega}}^{(l)}$ refers to the right singular vector of $\bm{C}$ such that $\Pi_{(\bm{C})} = \sum_l \bm{\omega}^{(l)} \bm{\omega}^{(l)\top}$ and $\bm{\bm{\omega}}^{(i)\top}\bm{C}^{\top} \bm{C} \bm{\omega}^{(j)}=\delta_{ij}(\sigma^{(i)})^2$, where $(\sigma^{(i)})^2\geq 4/(5\kappa^2)$. Suppose that the rank of both $R$ and $C$ is $k$ and $$\|\bm{R}^{\top}\bm{R}-\bm{C}^{\top}\bm{C}\|_2\leq \gamma~.$$ Let $\tilde{\bm{v}}^{(l)}:=\bm{R}\bm{\omega}^{(l)}/\sigma^{(l)}$, then we have 
	\begin{equation}\label{eqn:lem3.3}
		\left|\tilde{\bm{v}}^{(i)\top}\tilde{\bm{v}}^{(j)}-\delta_{ij}\right|  \leq  \frac{5\kappa^2\gamma}{4}~.
	\end{equation}
\end{lem}
The proof of Lemma \ref{lem:3.3} is presented in Subsection \ref{subsec:Lemma14}.

In conjunction with  Eqn.~(\ref{eqn:thm1-1}) and Eqn.~(\ref{eqn:lem3.3}),  we set  $\alpha =  {(5\gamma\kappa^2)}/{4}$ and rewrite  Eqn.~(\ref{eqn:alpha_error}) as  
\begin{equation}\label{eqn:inq_gamma}
	\left\| \sum_{i=1}^k \tilde{\bm{v}}^{(i)}\tilde{\bm{v}}^{(i)\top} - \Pi_{(\bm{H})} \right\|_2 \leq \frac{85k\gamma\kappa^2}{12}\leq \frac{\epsilon}{4}~.
\end{equation} 
In other words, when  $\gamma\leq \frac{3\epsilon}{85k\kappa^2}$, Eqn.~(\ref{eqn:thm_13_core}) is achieved so that $\|\bm{\tilde{H}}_t-\bm{H}_t\| \leq \frac{\epsilon}{4}\|\bm{H}_t\|$. Recall that $\gamma$ is quantified by $\|\bm{R}^{\top}\bm{R}-\bm{C}\bm{C}^{\top}\|_2$ as defined in Eqn.~(\ref{eqn:lem3.3}), we use the following theorem to bound $\gamma$, i.e., 
\begin{thm}\label{thm:sample_Alg}
Given a nonnegative matrix $\bm{H}\in\mathbb{R}^{n\times m}$, let $\bm{R}\in\mathbb{R}^{n\times s}$, $\bm{C}\in\mathbb{R}^{s\times  s}$ be the sampled matrix following  Algorithm \ref{alg:subsamp}. Setting $s$ as $s=\frac{85^2k^2\kappa^4\ln(8n/\eta)\|\bm{H}\|_F^2}{9\epsilon^2 }$, with probability at least $(1-\eta)$, we always have $\|\bm{R}^{\top}\bm{R}-\bm{C}^{\top}\bm{C}\|\leq \gamma$, 
\end{thm}
The proof of Theorem \ref{thm:sample_Alg} is given in \ref{subsec:C.1.3}.

Combining the result of Theorem \ref{thm:main} and Lemma \ref{lem:3.3}, we know that with sampling $s$ rows of $\bm{H}$,  the approximated distribution  is $\frac{\epsilon}{2}$-close to the desired result, i.e.,  $$\|\mathcal{P}_{\bm{\tilde{H}}_t}-\mathcal{P}_{\bm{H}_t}\|_{TV}\leq \frac{\epsilon}{2} ~.$$

\subsubsection{Proof of Theorem \ref{thm:main}} \label{subsec:D1}       
 We first introduce a lemma to facilitate the proof of Theorem \ref{thm:main}, i.e., 
\begin{lem}[Adapted from Lemma 5, \cite{gilyen2018quantum}]\label{lem:specB}
	Let $\bm{A}$ be a matrix of rank at most $k$, and suppose that $\bm{W}$ has $k$ columns that span the row and column spaces of $\bm{A}$. Then $\|\bm{A}\|_2\leq \|(\bm{W}^{\top}\bm{W})^{-1}\|_2\|\bm{W}^{\top}\bm{AW} \|_2$~.
\end{lem}  
\begin{proof}[Proof of Theorem \ref{thm:main}]
The main procedure to prove this theorem is as follows. By employing the Lemma \ref{lem:specB}, we can set $\bm{A}$ and $\bm{W}$ as
 $$\bm{A}:=\sum_{i=1}^k \tilde{\bm{v}}^{(i)}\tilde{\bm{v}}^{(i)\top} - \Pi_{(\bm{H})} ~, ~\bm{W}\equiv \bm{V}=\sum_{i=1}^k\tilde{\bm{v}}^{(i)}\tilde{\bm{v}}^{(i)\top}~,$$
 and then bound $\|\bm{W}^{\top}\bm{AW} \|_2$ and $ \|(\bm{W}^{\top}\bm{W})^{-1}\|_2$ separately. Lastly, we combine the two results to obtain the bound $	\| \sum_{i=1}^k \tilde{\bm{v}}^{(i)}\tilde{\bm{v}}^{(i)\top} - \Pi_{(\bm{H})} \|_2 \leq   \frac{ 17k\alpha}{3}$  in Eqn.~(\ref{eqn:thm1}).	 

Following the above observation, we first bound the term $\|\bm{W}^{\top}\bm{AW}\|_2$. We rewrite $\bm{W}^{\top}\bm{AW}$ as 
\begin{equation}\label{eqn:C.1.1}
\bm{W}^{\top}\bm{AW} = \sum_{i,j=1}^k \tilde{\bm{v}}^{(i)}(\tilde{\bm{v}}^{(i)\top}A\tilde{\bm{v}}^{(j)})\tilde{\bm{v}}^{(j)\top} ~.
\end{equation}
  The entry $\bm{A}(i,j)$ of $\bm{A}$ with $\bm{A}(i,j)=(\tilde{\bm{v}}^{(i)\top}A\tilde{\bm{v}}^{(j)})$ is bounded by  $\alpha$, i.e.,  
\begin{eqnarray}\label{eqn:thm2}
&& \left|\tilde{\bm{v}}^{(i)\top} \bm{A}\tilde{\bm{v}}^{(j)}\right|\nonumber\\
	=&&\left|  \tilde{\bm{v}}^{(i)\top}\left(\sum_{t=1}^k \tilde{\bm{v}}^{(t)}\tilde{\bm{v}}^{(t)\top} - \Pi_{(\bm{H})}	\right)\tilde{\bm{v}}^{(j)}\right| \nonumber\\
	= && \left| \sum_{t=1}^k  \tilde{\bm{v}}^{(i)\top} \tilde{\bm{v}}^{(t)}\tilde{\bm{v}}^{(t)\top}\tilde{\bm{v}}^{(j)} - \tilde{\bm{v}}^{(i)\top}\tilde{\bm{v}}^{(j)}  \right|\nonumber\\
	\leq && \left| \sum_{t=1}^k  \tilde{\bm{v}}^{(i)\top} \tilde{\bm{v}}^{(t)}\tilde{\bm{v}}^{(t)\top}\tilde{\bm{v}}^{(j)} -  \delta_{ij}\right| +\alpha \nonumber\\ \leq && 
	 \left| \sum_{t=1,t\neq  \{i,j\}}^k  \tilde{\bm{v}}^{(i)\top} \tilde{\bm{v}}^{(t)}\tilde{\bm{v}}^{(t)\top}\tilde{\bm{v}}^{(j)}\right| + \left| \sum_{t'= \{i,j\}}\tilde{\bm{v}}^{(i)\top} \tilde{\bm{v}}^{(t')}\tilde{\bm{v}}^{(t')\top}\tilde{\bm{v}}^{(j)}- \delta_{ij}\right| +\alpha \nonumber\\
	\leq &&  \left| \sum_{t=1,t\neq \{i,j\}}^k  \tilde{\bm{v}}^{(i)\top} \tilde{\bm{v}}^{(t)}\tilde{\bm{v}}^{(t)\top}\tilde{\bm{v}}^{(j)}\right| +4\alpha \nonumber\\
	 \leq && \frac{17\alpha}{4} ~. 
\end{eqnarray} 

The first equivalence of Eqn.~(\ref{eqn:thm2}) comes from the definition of $\bm{A}$, and the second equivalence employs $\tilde{\bm{v}}^{(i)\top}\Pi_{(\bm{H})}\tilde{\bm{v}}^{(j)} = \tilde{\bm{v}}^{(i)\top}\tilde{\bm{v}}^{(j)}.$ The first inequality of Eqn.~(\ref{eqn:thm2}) exploits triangle inequality and Eqn.~(\ref{eqn:thm1-1}), i.e., 
\begingroup
\allowdisplaybreaks
\begin{align}
&  \left| \sum_{t=1}^k  \tilde{\bm{v}}^{(i)\top} \tilde{\bm{v}}^{(t)}\tilde{\bm{v}}^{(t)\top}\tilde{\bm{v}}^{(j)} - \tilde{\bm{v}}^{(i)\top}\tilde{\bm{v}}^{(j)}\right|\nonumber\\ 
\leq & \left| \sum_{t=1}^k  \tilde{\bm{v}}^{(i)\top} \tilde{\bm{v}}^{(t)}\tilde{\bm{v}}^{(t)\top}\tilde{\bm{v}}^{(j)}-\delta_{ij}\right|+ \left|  \delta_{ij} - \tilde{\bm{v}}^{(i)\top}\tilde{\bm{v}}^{(j)}\right|\nonumber\\
 \leq & \left| \sum_{t=1}^k  \tilde{\bm{v}}^{(i)\top} \tilde{\bm{v}}^{(t)}\tilde{\bm{v}}^{(t)\top}\tilde{\bm{v}}^{(j)}-\delta_{ij}\right|+ \alpha    ~. 
\end{align}
\endgroup  
The second inequality of Eqn.~(\ref{eqn:thm2}) directly comes from the triangle inequality. The last second inequality of Eqn.~(\ref{eqn:thm2}) employs the inequality $\left| \sum_{t'= \{i,j\}}\tilde{\bm{v}}^{(i)\top} \tilde{\bm{v}}^{(t')}\tilde{\bm{v}}^{(t')\top}\tilde{\bm{v}}^{(j)}- \delta_{ij}\right|\leq 3\alpha$ for both the case $i=j$ and $i\neq j$,  guaranteed by Eqn.~(\ref{eqn:thm1-1}) and $\alpha^2<\alpha$. Specifically,  for the case $i\neq j$, we bound $\left| \sum_{t'= \{i,j\}}\tilde{\bm{v}}^{(i)\top} \tilde{\bm{v}}^{(t')}\tilde{\bm{v}}^{(t')\top}\tilde{\bm{v}}^{(j)}- \delta_{ij}\right|$ as
\begingroup
\allowdisplaybreaks
\begin{align}
&\left| \sum_{t'= \{i,j\}}\tilde{\bm{v}}^{(i)\top} \tilde{\bm{v}}^{(t')}\tilde{\bm{v}}^{(t')\top}\tilde{\bm{v}}^{(j)}\right| 
\leq  \left|\tilde{\bm{v}}^{(i)\top} \tilde{\bm{v}}^{(j)}\tilde{\bm{v}}^{(j)\top}\tilde{\bm{v}}^{(j)}\right|+\left|\tilde{\bm{v}}^{(i)\top} \tilde{\bm{v}}^{(i)}\tilde{\bm{v}}^{(i)\top}\tilde{\bm{v}}^{(j)}\right|  \nonumber\\
= & \left |\tilde{\bm{v}}^{(i)\top} \tilde{\bm{v}}^{(j)}\right|\left|\tilde{\bm{v}}^{(j)\top}\tilde{\bm{v}}^{(j)}\right|+\left|\tilde{\bm{v}}^{(i)\top} \tilde{\bm{v}}^{(i)}\right|\left|\tilde{\bm{v}}^{(i)\top}\tilde{\bm{v}}^{(j)}\right|\nonumber\\
\leq &  \alpha\left(\left|\tilde{\bm{v}}^{(j)\top}\tilde{\bm{v}}^{(j)}-\delta_{jj}+\delta_{jj}\right|\right) +\left(\left|\tilde{\bm{v}}^{(i)\top}\tilde{\bm{v}}^{(i)}-\delta_{ii}+\delta_{ii}\right|\right)  \alpha  \nonumber\\
\leq & 2\alpha(\alpha+1)
\leq  3\alpha~. \nonumber
\end{align}
\endgroup  
For the case $i=j$, we bound $\left| \sum_{t'= \{i,j\}}\tilde{\bm{v}}^{(i)\top} \tilde{\bm{v}}^{(t')}\tilde{\bm{v}}^{(t')\top}\tilde{\bm{v}}^{(j)}- \delta_{ij}\right|$ as 
\begingroup
\allowdisplaybreaks
\begin{align}
	&\left| \sum_{t'= \{i,j\}}\tilde{\bm{v}}^{(i)\top} \tilde{\bm{v}}^{(t')}\tilde{\bm{v}}^{(t')\top}\tilde{\bm{v}}^{(j)}- \delta_{ij}\right|= \left|\tilde{\bm{v}}^{(i)\top} \tilde{\bm{v}}^{(i)}\tilde{\bm{v}}^{(i)\top}\tilde{\bm{v}}^{(i)}- \delta_{ii}\right|\nonumber\\
	\leq & \left|\tilde{\bm{v}}^{(i)\top} \tilde{\bm{v}}^{(i)}- \delta_{ii}\right|\left|\tilde{\bm{v}}^{(i)\top}\tilde{\bm{v}}^{(i)}+\delta_{ii}\right|\leq \alpha(\alpha+1)\leq 3\alpha ~. \nonumber
\end{align}
\endgroup
 The last inequality of Eqn.~(\ref{eqn:thm2}) comes from $$\left| \sum_{t=1,t\neq \{i,j\}}^k  \tilde{\bm{v}}^{(i)\top} \tilde{\bm{v}}^{(t)}\tilde{\bm{v}}^{(t)\top}\tilde{\bm{v}}^{(j)}\right|\leq \frac{\alpha}{4}~,$$ since 
 \begingroup
\allowdisplaybreaks
\begin{align}
	&\left| \sum_{t=1,t\neq \{i,j\}}^k  \tilde{\bm{v}}^{(i)\top} \tilde{\bm{v}}^{(t)}\tilde{\bm{v}}^{(t)\top}\tilde{\bm{v}}^{(j)}\right| =   \sum_{t=1,t\neq \{i,j\}}^k\left|\tilde{\bm{v}}^{(i)\top} \tilde{\bm{v}}^{(t)}\right|^2\leq  \sum_{t=1,t\neq \{i,j\}}^k \alpha^2 \leq  k\alpha^2 \leq \alpha/4~. \nonumber
\end{align}
\endgroup

Combing Eqn.~(\ref{eqn:C.1.1}) and Eqn.~(\ref{eqn:thm2}),  we immediately have $$\left\|\bm{W}^{\top}\bm{AW}\right\|_2\leq \frac{17k\alpha}{4}~.$$
In addition, from Eqn.~(\ref{eqn:thm1-1}) and the definition of $\bm{W}$, we can obtain $$\left\|\bm{W}^{\top}\bm{W}-\bm{I}\right\|_2\leq k\alpha\leq 1/4$$  and then $\left\|(\bm{W}^{\top}\bm{W})^{-1}\right\|_2\leq 4/3$. Combining the two result, we have 
\begin{equation}\label{eqn:bound_alpha}
	\|\bm{A}\|_2\leq \left\|(\bm{W}^{\top}\bm{W})^{-1}\right\|_2\left\|\bm{W}^{\top}\bm{AW} \right\|_2 \leq \frac{17k\alpha}{3}~.
\end{equation}
\end{proof}

\subsubsection{Proof of Lemma \ref{lem:3.3}}\label{subsec:Lemma14}
\begin{proof}[Proof of Lemma \ref{lem:3.3}]
	The inequality $\left|\tilde{\bm{v}}^{(i)\top}\tilde{\bm{v}}^{(j)}-\delta_{ij}\right|  \leq  \frac{5\kappa^2\gamma}{4}$ in Eqn.~(\ref{eqn:lem3.3}) can be proved following the definition of $\tilde{\bm{v}}^{(i)}$. Mathematically, we have 
	\begin{eqnarray}\label{eqn:lem3.3-3}
&&\left|\tilde{\bm{v}}^{(i)\top}\tilde{\bm{v}}^{(j)}-\delta_{ij}  \right|
		= 	\left| \frac{\bm{\omega}^{(i)\top}\bm{R}^{\top}\bm{R} \bm{\omega}^{(j)}}{\sigma^{(i)}\sigma^{(j)}} -\delta_{ij} \right| \nonumber\\
		\leq &&   
		\left| \frac{ \bm{\omega}^{(i)\top}  \bm{C}^{\top}\bm{C} \bm{\omega}^{(j)} }{\sigma^{(i)}\sigma^{(j)}} -\delta_{ij} \right|	+ \frac{\gamma}{\sigma^{(i)}\sigma^{(j)}}\nonumber\\
		 =&& \frac{\gamma}{\sigma^{(i)}\sigma^{(j)}} \leq\frac{5\kappa^2\gamma}{4}~.
	\end{eqnarray}
The first equivalence of Eqn.~(\ref{eqn:lem3.3-3}) comes from the definition of $\tilde{\bm{v}}^{(i)}$. The first  inequality of Eqn.~(\ref{eqn:lem3.3-3}) is derived by employing $\|\bm{R}^{\top}\bm{R}-\bm{C}^{\top}\bm{C}\|\leq \gamma$, i.e., 
 \begingroup
\allowdisplaybreaks
\begin{align}
	&\left| \frac{\bm{\omega}^{(i)\top}\bm{R}^{\top}\bm{R} \bm{\omega}^{(j)}}{\sigma^{(i)}\sigma^{(j)}} -\delta_{ij}\right| 
\nonumber\\
	= & \left| \frac{\bm{\omega}^{(i)\top}(\bm{R}^{\top}\bm{R} -\bm{C}^{\top}\bm{C}+\bm{C}^{\top}\bm{C})\bm{\omega}^{(j)}}{\sigma^{(i)}\sigma^{(j)}}-\delta_{ij}\right|\nonumber\\
	\leq & \left| \frac{\bm{\omega}^{(i)\top}(\bm{R}^{\top}\bm{R} -\bm{C}^{\top}\bm{C})\bm{\omega}^{(j)}}{\sigma^{(i)}\sigma^{(j)}}\right|+\left| \frac{\bm{\omega}^{(i)\top}\bm{C}^{\top}\bm{C}\bm{\omega}^{(j)}}{\sigma^{(i)}\sigma^{(j)}}-\delta_{ij}\right|  \nonumber\\
	\leq & \frac{\gamma}{\sigma^{(i)}\sigma^{(j)}}+\left| \frac{\bm{\omega}^{(i)\top}\bm{C}^{\top}\bm{C}\bm{\omega}^{(j)}}{\sigma^{(i)}\sigma^{(j)}}-\delta_{ij}\right| ~.
\end{align}
\endgroup
 The last second equivalence of Eqn.~(\ref{eqn:lem3.3-3}) employs $$\bm{\omega}^{(i)\top}\bm{C}^{\top} \bm{C} \bm{\omega}^{(j)}=\delta_{ij}(\sigma^{(i)})^2~.$$ The last inequality  of Eqn.~(\ref{eqn:lem3.3-3}) uses  $$(\sigma^{(i)})^2\geq 4/(5\kappa^2)~.$$
\end{proof}

\subsubsection{Proof of Theorem \ref{thm:sample_Alg}}\label{subsec:C.1.3}
We  introduce the following lemma to facilitate the proof.  
\begin{lem}[Adapted from Theorem 4.4, \cite{kannan2017randomized}]\label{lem:sample_alg}
	Given any matrix $\bm{R}\in \mathbb{R}^{n\times s}$. Let ${\bm{C}} \in\mathbb{R}^{s\times s}$ be obtained by length-squared sampling with $\mathbb{E}( \bm{C}^{\top}\bm{C}) = {\bm{R}}^{\top}{\bm{R}}$. Then, for all $\epsilon\in[0, \|\bm{R}\|_2/\|\bm{R}\|_F]$, we have 
	\begin{equation}
		\pr\left(\left\|\bm{C}^{\top}\bm{C}-{\bm{R}}^{\top}{\bm{R}}\right\|_2\geq \epsilon\|\bm{R}\|_2\|\bm{R}\|_F \right)\leq 2ne^{-\epsilon^2s/4}~.
	\end{equation} 
Hence, for $s\geq 4\ln{(2n/\eta)/\epsilon^2}$, with probability at least $(1-\eta)$ we have $\left\|\bm{C}^{\top}\bm{C}-{\bm{R}}^{\top}{\bm{R}}\right\|_2\leq \epsilon\|\bm{R}\|\|\bm{R}\|_F$. 
\end{lem} 
\begin{proof}[Proof of Theorem \ref{thm:sample_Alg}]
The Lemma \ref{lem:sample_alg} indicates that the sample complexity of $s$ determines $\|\bm{C}^{\top}\bm{C}-{\bm{R}}^{\top}{\bm{R}}\|_2$. With setting $\gamma= \epsilon\|\bm{R}\|_2\|\bm{R}\|_F$, we have 
\begin{eqnarray}
\pr(\left\|\bm{C}^{\top}\bm{C}-{\bm{R}}^{\top}{\bm{R}}\right\|_2\geq \gamma)\leq 2ne^{-\gamma^2s/4(\|\bm{R}\|_2\|\bm{R}\|_F)}
\end{eqnarray}
Let the right hand side of the above inequality be $\eta$, i.e.,
\begin{eqnarray}
&&	2ne^{-\gamma^2s/4(\|\bm{R}\|_2\|\bm{R}\|_F)}={\eta}\nonumber\\
\xrightarrow{log}&&4\|\bm{R}\|_2\|\bm{R}\|_F\ln(2n/\eta)=\gamma^2 s  \nonumber\\
\rightarrow && s= \frac{4\|\bm{R}\|_2\|\bm{R}\|_F\ln(2n/\eta)}{\gamma^2}\nonumber\\
\rightarrow && s=\frac{85^2k^2\kappa^4\|\bm{R}\|_2\|\bm{R}\|_F\ln(8n/\eta)}{9\epsilon^2 }\nonumber\\
&& \leq 
\frac{85^2k^2\kappa^4\|\bm{H}\|_F^2\ln(8n/\eta)}{9\epsilon^2 }~,	
\end{eqnarray}
where the inequality comes  from $\|\bm{R}\|_F\leq  \|\bm{R}\|_F$ and  $\|\bm{R}\|_F= \|\bm{H}\|_F$. Therefore,  with setting $s$ as $$s=\frac{85^2k^2\kappa^4\ln(8n/\eta)\|\bm{H}\|_F^2}{9\epsilon^2 }~,$$ we have $\|{\bm{R}}^{\top}{\bm{R}}-\bm{C}^{\top}\bm{C}\|_2\leq \gamma$ with probability at least $(1-\eta)$. 
\end{proof}

\subsection{Proof of Theorem  \ref{thm:total_va_aXb}}\label{subsec:C.2}\begin{proof}[Proof of Theorem  \ref{thm:total_va_aXb}] 
	 We first give the upper bound of the term $\left\|\hat{\bm{H}}_t -  \tilde{\bm{H}}_t\right\| $, i.e., 
 \begin{equation}\label{eqn:lem37}
\left\|\hat{\bm{H}}_t -  \tilde{\bm{H}}_t\right\|=\left\|\tilde{\bm{V}}\bm{\hat{q}}_{{\scalebox{.49}{H,t}}}-   \tilde{\bm{V}}\bm{\tilde{q}}_{{\scalebox{.49}{H,t}}}  \right\| \leq  \left\|\tilde{\bm{V}}\right\|_2\left\|\bm{\hat{q}}_{{\scalebox{.49}{H,t}}}-\bm{\tilde{q}}_{{\scalebox{.49}{H,t}}}\right\|\leq  \frac{4+\epsilon}{4}\left\|\bm{\hat{q}}_{{\scalebox{.49}{H,t}}}-\bm{\tilde{q}}_{{\scalebox{.49}{H,t}}}\right\|~.
 \end{equation} 
 The first inequality comes from the the submultiplicative property of spectral norm. The second inequality supported by Eqn.~(\ref{eqn:inq_gamma}) with 
 $$\|\bm{\tilde{V}}\|_2\leq 1+\frac{\epsilon}{4}$$
Following the definition of $\ell_2$ norm, we have $$\left\|\bm{\hat{q}}_{{\scalebox{.49}{H,t}}}- \bm{\tilde{q}}_{{\scalebox{.49}{H,t}}}\right\|^2=\sum_{i=1}^k (\bm{\hat{q}}_{{\scalebox{.49}{H,t}}}{(i)}-\bm{\tilde{q}}_{{\scalebox{.49}{H,t}}}(i))^2~.$$ Denote the additive error $$ \epsilon' = \max_{i\in[k]} |\bm{\hat{q}}_{{\scalebox{.49}{H,t}}}(i)-\bm{\tilde{q}}_{{\scalebox{.49}{H,t}}}(i)| ~,$$ we rewrite  Eqn.~(\ref{eqn:lem37})  as   
 	\begin{equation}
 	\left\|\hat{\bm{H}}_t -  \tilde{\bm{H}}_t\right\| \leq   \frac{4+\epsilon}{4}\sqrt{k}\epsilon'~. 	
 	\end{equation}
 An observation of the above equation is that
 we have  $\left\|\hat{\bm{H}}_t -  \tilde{\bm{H}}_t\right\| \leq \epsilon/4$ if
 \begin{equation}\label{eqn:SM_core_aXb}
 	\epsilon'\leq\frac{4\epsilon}{(4+\epsilon)\sqrt{k}}~.
 \end{equation}
 
 We use the result of the inner product subroutine to quantify the required number of samplings to achieve  Eqn.~(\ref{eqn:SM_core_aXb}).   The conclusion of Lemma \ref{lem:aXb} is that when $$N_Z\sim \mathcal{O}\left(\frac{\|\bm{H}\|_F\|\bm{B}_t\|\|\tilde{\bm{v}}^{(i)}\|}{\epsilon}\log(1/\delta)\right)~,$$ we have  $|\bm{\tilde{q}}_{{\scalebox{.49}{H,t}}}(i) - \bm{\hat{q}}_{{\scalebox{.49}{H,t}}}(i)|\leq \epsilon$ with at least $1-\delta$ success probability. With substituting $\epsilon$ by $\epsilon'$,  we immediately obtain  $|\bm{\tilde{q}}_{{\scalebox{.49}{H,t}}}(i) - \bm{\hat{q}}_{{\scalebox{.49}{H,t}}}(i)|\leq \epsilon'$, where the required number of samplings is 
 \begin{equation}\label{eqn:SM_cor_aXb_conc}
 	N_Z\sim \mathcal{O}\left(\frac{(4+\epsilon)\sqrt{k}\|\bm{H}\|_F\|\bm{B}_t\|\|\tilde{\bm{v}}^{(i)}\|}{4\epsilon}\log(1/\delta)\right)~.
 \end{equation}
 \end{proof}

\section{ The Complexity of The Algorithm (Proof of Theorem \ref{thm:comp})}\label{sec:SM_compl}
\begin{proof}[Proof of Theorem \ref{thm:comp}]	
As analyzed in the main text, the complexity of the proposed algorithm is dominated by four operations in the preprocessing step and the divide step, i.e., finding the left singular vectors $\tilde{\bm{V}}$,  estimating the inner product to build  $\hat{\bm{q}}_t$, preparing the approximated probability distribution $\mathcal{P}_{\hat{\bm{H}}}$, and estimating the rescale factor $\hat{\xi}_t$.  We  evaluate the computation complexity of these four operations separately and then give the overall computation complexity of our algorithm.  

In this subsection, we first evaluate the computation complexity of these four  parts separately  and then combine the results to give the computation complexity of our algorithm.  Due to same reconstruction rule, we use a general setting $\bm{H}\in\mathbb{R}^{n\times m}$ that can either be $\bm{X}$ or $\bm{Y}$ to evaluate the computation complexity for the four parts.

\textit{\underline{Complexity of Finding $\tilde{\bm{V}}$}}. 
Supported by  the $\ell_2$ norm sampling operations, the matrix $\bm{C}$ can be efficiently constructed following  Algorithm \ref{alg:subsamp}, where $\mathcal{O}(2s\log^2(mn))$ query complexity is sufficient. Applying SVD onto $\bm{C}\in\mathbb{R}^{s\times s}$ with $s=\frac{85^2k^2\kappa^4\ln(8n/\eta)\|\bm{H}\|_F^2}{9\epsilon^2 }$ generally costs    $$\mathcal{O}({s^3})=\tilde{\mathcal{O}}\left(\frac{85^6k^6\kappa^{12}\|\bm{H}\|_F^6}{9^3\epsilon^6 }\right)$$ runtime complexity.  Once we obtain such the SVD result of $\bm{C}$, the approximated left singular vectors $\tilde{\bm{V}}$ can be implicitly represented, guaranteed by the following Lemma:
\begin{lem}[Adapted from \cite{tang2018quantum}]\label{lem:Q_S_V}
Let the given dataset support the  $\ell_2$ norm sampling operations  along with the description of  $\tilde{\bm{V}}\in \mathbb{R}^{n\times s}$, We can sample from any $\tilde{\bm{v}}^{(t)}$ in $\mathcal{O}(Ks^2)$ expected queries with $K=\kappa\|\bm{H}\|_F^2$ and query for any particular entry $\tilde{\bm{V}}(i,j)$ in $\mathcal{O}(s)$ queries.	
\end{lem}

\textit{\underline{Complexity of Estimating $\bm{\tilde{q}}_t$ by $\bm{\hat{q}}_t$}}. 
The runtime complexity to estimate  $\bm{\tilde{q}}_t$ by $\bm{\hat{q}}_t$ obeys the following corollary, i.e., 
\begin{coro}\label{thm:total_va_aXb1} 
Let $\bm{B}_t\in\mathbb{R}^{m\times 1}$ be the input vector, $\bm{H}\in\mathbb{R}^{n\times m}$ be the input matrix with rank $k$, and $\bm{\tilde{V}}\in \mathbb{R}^{n\times k}$ be the approximated left singular matrix.  We can estimate $\tilde{\bm{q}}_t=\tilde{\bm{V}}^{\top}\bm{H}\bm{B}_t$ by $\hat{\bm{q}}_t$ to precision $\epsilon$ with probability at least $1-\delta$ using  $$\mathcal{O}\left(\frac{4 k(1+\epsilon)^{1.5}\|\bm{H}\|_F}{\epsilon}(\log^2(mn)+\frac{85^2k^2\kappa^{4}\|\bm{H}\|_F^2\ln{(8n/\eta)}}{9\epsilon^2 })\log(\frac{1}{\delta})\right)$$
 runtime complexity. 
 \end{coro}
 
 \begin{proof}[Proof of Corollary \ref{thm:total_va_aXb1}]
 The proof of Corollary  \ref{thm:total_va_aXb1} employs the result of Theorem \ref{thm:total_va_aXb} and Lemma \ref{lem:aXb}.

 Theorem \ref{thm:total_va_aXb} indicates that, to estimate $\tilde{\bm{q}}_t(i)$ by $\hat{\bm{q}}_t(i)$,   the required number of samplings is
 $$	N_Z\sim \mathcal{O}\left(\frac{(4+\epsilon)\sqrt{k}\|\bm{H}\|_F\|\bm{B}_t\|\|\tilde{\bm{v}}^{(i)}\|}{4\epsilon}\log(1/\delta)\right)~.$$ 
 Following the result of Lemma \ref{lem:aXb}, the runtime complexity  to obtain $\hat{\bm{q}}_t(i)$  is  \begin{equation}\label{eqn:SM_comp_aXb}
 N_Z(L(\bm{H})+Q(\bm{\tilde{V}})+Q(\bm{B}_t))~.	
 \end{equation}

Since  $\|\tilde{\bm{v}}^{(i)}\|\leq \sqrt{1+\epsilon}$ for any $\tilde{\bm{v}}^{(i)}$ indicated by Eqn.~(\ref{eqn:lem37}),  we rewrite Eqn.~(\ref{eqn:SM_comp_aXb}) as   
\begin{eqnarray}
&&	N_Z(L(\bm{H})+Q(\bm{\tilde{V}}_t)+Q(\bm{B}_t))\nonumber\\
\leq && \mathcal{O}\left(\frac{\|\bm{H}\|_F({1+\epsilon})^{1.5}\sqrt{k}}{\epsilon}(L(\bm{H})+Q(\bm{\tilde{V}}_t)+Q(\bm{B}_t))\log(1/\delta)\right)~.
\end{eqnarray}
We now quantify the  access cost of $\bm{H}$, $\bm{B}_t$ and  $\bm{\tilde{V}}$ to give an explicit bound of Eqn.~(\ref{eqn:SM_comp_aXb}). We have $L(\bm{H})=\mathcal{O}(\log^2(nm))$ and $Q(\bm{B}_t)=\mathcal{O}(\log(m))$, since $\bm{H}$ and $\bm{B}_t$ are stored in BNS data structure. We have $Q(\bm{\tilde{V}})=\mathcal{O}(s)$ supported by Lemma \ref{lem:Q_S_V}. Combing the above access cost and Eqn.~(\ref{eqn:SM_comp_aXb}), the runtime complexity to estimate $\hat{\bm{q}}_t(i)$ is  
  $$\mathcal{O}\left(\frac{(1+\epsilon)^{1.5}\sqrt{k}\|\bm{H}\|_F}{\epsilon}(\log^2(mn)+\frac{85^2k^2\kappa^{4}\|\bm{H}\|_F^2\ln{(8n/\eta)}}{9\epsilon^2 })\log(\frac{1}{\delta})\right) ~.$$
 Since each entry can be computed in parallel, the runtime complexity to obtain $\bm{\hat{q}}_t$ is also   $$\mathcal{O}\left(\frac{(1+\epsilon)^{1.5}\sqrt{k}\|\bm{H}\|_F}{\epsilon}(\log^2(mn)+\frac{85^2k^2\kappa^{4}\|\bm{H}\|_F^2\ln{(8n/\eta)}}{9\epsilon^2 })\log(\frac{1}{\delta})\right) ~.$$
\end{proof}

\textit{\underline{Complexity of Sampling from $\mathcal{P}_{\bm{\hat{H}}_t}$.}} Recall that the definition of  ${\bm{\tilde{H}}_t}$ is ${\bm{\hat{H}}_t}=\tilde{V}\bm{\hat{q}}_t$. We first evaluate the computation complexity to obtain one sample from $\mathcal{P}_{\bm{\hat{H}}_t}$. From Lemma \ref{lem:Ab}, we know the expected  runtime complexity to sample from $\mathcal{P}_{\bm{\hat{H}}_t}$  is $\mathcal{O}(\frac{k\|\bm{\hat{q}}_t\|^2}{\|\tilde{\bm{V}}\bm{\hat{q}}_t\|^2}(S(\tilde{\bm{V}})+kQ(\tilde{\bm{V}})))$. Specifically, we have 
\begin{equation}\label{eqn:bound_q_t}
\|\bm{\hat{q}}_t\|\leq 	\|\bm{q}_t\| +\frac{\epsilon}{2} = \|\bm{\tilde{V}}^{\top}\bm{H}_t\| +\frac{\epsilon}{2} \leq  \left\|\bm{\tilde{V}}^{\top}\right\|_2 \|\bm{H}_t\| +\frac{\epsilon}{2} \leq \sqrt{(1+\epsilon)}\|{\bm{H}\|_F}+\frac{\epsilon}{2}~,
\end{equation} 
where the first inequality employs the triangle inequality, the second inequality employs the submultiplicative property of spectral norm,  and the third  inequality utilizes $\|\bm{H}_t\|\leq \|\bm{H}\|_2\|\bm{B}_t\|\leq \|\bm{H}\|_F\|\bm{B}_t\|$ with $\|\bm{B}_t\|=1$. Concurrently, we have $\|\tilde{\bm{V}}\bm{q}_t\|=\Omega(1)$. Employing Lemma \ref{lem:Q_S_V} to quantify $S(\tilde{\bm{V}})$ and $Q(\tilde{\bm{V}})$, the complexity to obtain a sample from $\mathcal{P}_{\bm{\hat{H}}_t}$  is $$\mathcal{O}( k(1+\epsilon) ({\|\bm{H}\|_F^2\kappa s^2}+ks))~.$$ With substituting $s$ with its explicit representation 
in Theorem \ref{thm:main_corect}, the complexity is 
$$\mathcal{O}( {k(1+\epsilon)}({\|\bm{H}\|_F^2\kappa s^2}+ks)) \approx \mathcal{O}\left( \frac{85^4k^5\kappa^{9}\ln^2(8n/\eta)\|\bm{H}\|_F^6}{9^2\epsilon^4 } \right)~.$$  

Following the result of the general heuristic post-selection method in Theorem  \ref{thm:post_sele1}, we sample the distribution $\mathcal{P}_{\bm{\hat{H}}_t}$  with $N \sim \mathcal{\tilde{O}}(\frac{1}{\varepsilon})$ times in parallel, which gives the  runtime complexity $$\mathcal{O}\left( \frac{85^4k^5\kappa^{9}\ln^2(8n/\eta)\|\bm{H}\|_F^6}{9^2\epsilon^5 } \log^2{(nm)}\right)~.$$ 

\textit{\underline{Complexity of estimating $\xi_t$ by $\hat{\xi}_t$.}} For the general case with $\bm{X}\neq \bm{Y}$, we should estimate the $\ell_2$ norm of their project results, i.e., $\|\bm{H}_t\|^2 = \bm{H}_t^{\top}\bm{H}_t$ that $\bm{H}$ can either be $\bm{X}$ or $\bm{Y}$. 
Recall that the explicit representation of the approximated result is  $\|\bm{\hat{H}}_t\|$ with $\|\bm{H}_t\| = \bm{\hat{q}}_t^{\top}\bm{\tilde{V}}^{\top}\bm{\tilde{V}}\bm{\hat{q}}_t$, where the   $(j,j)$-th entry of $\bm{\tilde{V}}^{\top}\bm{\tilde{V}}$ is $\|\tilde{\bm{v}}^{(j)}\|^2$ and the else entries are zero. An immediate observation is that $\|\bm{H}_t\|$ can be obtained by using the inner product subroutine in Lemma \ref{lem:aXb}. 

We first calculate the sample and query complexity to query the  $(j,j)$-th entry of $\bm{\tilde{V}}^{\top}\bm{\tilde{V}}$, namely, the query and sample complexity to obtain the inner product of $\tilde{\bm{v}}^{(j)}$. By employing the result of the inner product subroutine, with removing $\bm{H}$ and setting both $\bm{B}_t$ and $\bm{\tilde{v}}^{(i)}$ as $\bm{\tilde{v}}^{(j)}$, we can  estimate $\tilde{\bm{v}}^{(j)\top}\tilde{\bm{v}}^{(j)}$ to precision $\epsilon$ with probability at least $1-\delta$ in time 
\begin{equation}\label{eqn:comp_xi_t_1}
	\mathcal{O}(\frac{\|\tilde{\bm{v}}^{(j)}\|^2}{\epsilon^2}(Q(\tilde{\bm{v}}^{(i)}))\log(\frac{1}{\delta}) )\leq\mathcal{O}(\frac{(1+\epsilon)}{\epsilon^2}(s))\log(\frac{1}{\delta}) )\approx  \mathcal{O}(\frac{85^2k^2\kappa^4\ln(8n/\eta)\|\bm{H}\|_F^2}{9\epsilon^4 }\log(\frac{1}{\delta}))~,
\end{equation}
where we use Lemma \ref{lem:Q_S_V} to get  $Q(\tilde{\bm{v}}^{(i)}) =\mathcal{O}(s) $ and  the inequality comes from $\|\tilde{\bm{v}}^{(i)}\|\leq (1+\epsilon)$ in Eqn.~(\ref{eqn:correct}). We store $k$ nonzero entries of $\bm{\tilde{V}}^{\top}\bm{\tilde{V}}$ in memory. 

We next use the inner product subroutine to obtain $\|\bm{\hat{H}}_t\|$ with  $\|\bm{\hat{H}}_t\| = \bm{\hat{q}}_t^{\top}\bm{\tilde{V}}^{\top}\bm{\tilde{V}}\bm{\hat{q}}_t$.   Since both $\bm{\hat{q}}_t$ and $\bm{\tilde{V}}^{\top}\bm{\tilde{V}}$ are stored in memory, we have $L(\bm{\tilde{V}}^{\top}\bm{\tilde{V}})=\mathcal{O}(k)$ and $Q(\bm{\hat{q}}_t)=\mathcal{O}(1)$. Following the result of Lemma \ref{lem:aXb}, we can  estimate $\|\hat{\bm{H}}_t\|$ to precision $\epsilon$ with probability at least $1-\delta$ with runtime complexity 
\begin{eqnarray}\label{eqn:comp_xi_2}
&&\mathcal{O}\left(\frac{\|\bm{\tilde{V}}^{\top}\bm{\tilde{V}}\|_F\|\bm{q}_t\|^2}{\epsilon^2}\log(\frac{1}{\delta}) \right)\nonumber \\
\leq &&\mathcal{O}\left(\frac{\sqrt{k(1+\epsilon)^2}\sqrt{k(1+\epsilon)\|\bm{H}\|_F}}{\epsilon^2}(k)\log(\frac{1}{\delta}) \right)~, 
\end{eqnarray} 
where the inequality comes from  $\|\tilde{\bm{v}}^{(i)}\|\leq (1+\epsilon)$ in Eqn.~(\ref{eqn:correct})  and Eqn.~(\ref{eqn:bound_q_t}).

Combining Eqn.~(\ref{eqn:comp_xi_t_1}) and Eqn.~(\ref{eqn:comp_xi_2}), the computation complexity to obtain the $\ell_2$ norm of $\|\hat{\bm{H}}_t\| $ is 
\begin{equation}\label{eqn:comp_xi_all}
	\mathcal{O}\left(\frac{85^2k^2\kappa^4\ln(8n/\eta)\|\bm{H}\|_F^2}{9\epsilon^4 }\log(\frac{1}{\delta})\right)~.
\end{equation}

\textit{\underline{The overall  complexity of our algorithm.}} An immediate observation of the above four parts is that the query complexity and runtime complexity of our algorithm is denominated by the  complexity of finding $\tilde{V}$, i.e., $\mathcal{O}({s^3})=\tilde{\mathcal{O}}(\frac{85^6k^6\kappa^{12}\|\bm{H}\|_F^6}{9^3\epsilon^6 })$ .   Since $\bm{H}$ is a general setting that can either be $\bm{X}$ or $\bm{Y}$, the runtime complexity for our algorithm is $$\max\left\{\tilde{\mathcal{O}}\left(\frac{85^6k_{\scalebox{.49}{X}}^6\kappa_{\scalebox{.49}{X}}^{12}\|\bm{X}\|_F^6}{9^3\epsilon^6 }\right), \tilde{\mathcal{O}}\left(\frac{85^6k_{\scalebox{.49}{X,t}}^6\kappa_{\scalebox{.49}{Y}}^{12}\|\bm{Y}\|_F^6}{9^3\epsilon^6 }\right) \right\}~.$$
\end{proof}

\end{document}